\def\Journaltemplate{0}

\ifnum \Journaltemplate=1
\let\LaTeXcline\cline \documentclass[default,Numbered]{sn-jnl}\let\cline\LaTeXcline
   
 \usepackage{manyfoot}
\usepackage{xcolor}
\usepackage{amsmath}
\else
	\documentclass[11pt]{scrartcl}
    \usepackage[square,numbers]{natbib}
    \providecommand{\keywords}[1]{\textbf{\textit{Keywords---}} #1}
\fi

\pdfoutput=1
\usepackage[utf8x]{inputenc}
\usepackage[T1]{fontenc}
\usepackage{lmodern}
\usepackage[ruled]{algorithm2e}
\usepackage{geometry}
\usepackage{subfig}
\usepackage{svg}
\usepackage{dirtytalk}
\usepackage{amsthm}
\usepackage{booktabs} 
\usepackage{siunitx} 
\usepackage{pgfplotstable} 
\usepackage{tikz}
\usetikzlibrary{svg.path}
\usepackage{tabularx,longtable,multirow,caption}
\usepackage{fncylab} 

\usepackage{url} 
\usepackage[english]{babel}
\usepackage{epstopdf} 
\usepackage{backref} 
\usepackage{array}
\usepackage{latexsym}

\newtheorem{theorem}{Theorem}

\newtheorem{proposition}{Proposition}

\usepackage{lastpage}
\usepackage{listings}
\usepackage[normalem]{ulem} 
\usepackage{textpos}

\usepackage{dsfont}
\usepackage{epstopdf} 
\usepackage{backref} 
\usepackage{array}
\usepackage{hyperref} 
\usepackage[colorinlistoftodos]{todonotes}

\usepackage{bm,bbm,amsfonts,nicefrac,latexsym,amsbsy,amscd,amsxtra,amsgen,amsopn,amsthm,amssymb}
\usepackage{lastpage}
\usepackage{enumerate}
\usepackage{fancyhdr}
\usepackage{mathrsfs}
\usepackage{enumitem}
\usepackage{scalerel}

\ifnum \Journaltemplate=0
\definecolor{orcidlogocol}{HTML}{A6CE39}
\tikzset{
	orcidlogo/.pic={
		\fill[orcidlogocol] svg{M256,128c0,70.7-57.3,128-128,128C57.3,256,0,198.7,0,128C0,57.3,57.3,0,128,0C198.7,0,256,57.3,256,128z};
		\fill[white] svg{M86.3,186.2H70.9V79.1h15.4v48.4V186.2z}
		svg{M108.9,79.1h41.6c39.6,0,57,28.3,57,53.6c0,27.5-21.5,53.6-56.8,53.6h-41.8V79.1z M124.3,172.4h24.5c34.9,0,42.9-26.5,42.9-39.7c0-21.5-13.7-39.7-43.7-39.7h-23.7V172.4z}
		svg{M88.7,56.8c0,5.5-4.5,10.1-10.1,10.1c-5.6,0-10.1-4.6-10.1-10.1c0-5.6,4.5-10.1,10.1-10.1C84.2,46.7,88.7,51.3,88.7,56.8z};
	}
}
\newcommand\orcidicon[1]{\href{https://orcid.org/#1}{\mbox{\scalerel*{
\begin{tikzpicture}[yscale=-1,transform shape]
\pic{orcidlogo};
\end{tikzpicture}}{|}}}}
\fi

\newtheorem{definition}{Definition}
\newcommand{\eg}{\textit{e.g.}}
\newcommand{\ie}{\textit{i.e.}}
\newcommand{\st}{\textrm{s.t.}}
\newcommand{\R}{\mathbb{R}}
\DeclareMathOperator*{\argmax}{arg\,max}
\DeclareMathOperator*{\argmin}{arg\,min}

\begin{document}
\ifnum \Journaltemplate=1
\title[Article Title]{A constrained optimization approach to improve robustness of neural networks}
\else
\title{A constrained optimization approach to improve robustness of neural networks}
\fi

\ifnum \Journaltemplate=1
\author[1]{\fnm{Shudian} \sur{Zhao} }\email{shudian@kth.se}
\author*[1]{\fnm{Jan} \sur{Kronqvist} }\email{jankr@kth.se}

\affil[1]{\orgdiv{Department of Mathematics}, \orgname{KTH - Royal Institute of Technology}, \orgaddress{\street{Lindtstedtsvägen 25}, \city{Stockholm}, \postcode{100 44}, \country{Sweden}}}
\else
\author{Shudian Zhao\footnote{KTH - Royal Institute of Technology, Stockholm,
Sweden, Department of Mathematics, Lindtstedtsvägen 25, SE-100 44 Stockholm; \href{mailto:shudian@kth.se}{shudian@kth.se}}~\orcidicon{0000-0001-6352-0968} \and Jan Kronqvist \footnote{KTH - Royal Institute of Technology, Stockholm,
Sweden, Department of Mathematics, Lindtstedtsvägen 25, SE-100 44 Stockholm; \href{mailto:jankr@kth.se}{jankr@kth.se}} ~\footnote{Corresponding Author}~\orcidicon{0000-0003-0299-5745}}
\date{\today}
\fi

\ifnum \Journaltemplate=1
\abstract{
  In this paper, we present a novel nonlinear programming-based approach to fine-tune pre-trained neural networks to improve robustness against adversarial attacks while maintaining high accuracy on clean data. Our method introduces adversary-correction constraints to ensure the correct classification of adversarial data and minimizes changes to the model parameters. We propose an efficient cutting-plane-based algorithm to iteratively solve the large-scale nonconvex optimization problem by approximating the feasible region through polyhedral cuts and balancing between robustness and accuracy. Computational experiments on standard datasets such as MNIST and CIFAR10 demonstrate that the proposed approach significantly improves robustness, even with a very small set of adversarial data, while maintaining minimal impact on accuracy. }

\maketitle
\else
\maketitle
\begin{abstract}
    In this paper, we present a novel nonlinear programming-based approach to fine-tune pre-trained neural networks to improve robustness against adversarial attacks while maintaining high accuracy on clean data. Our method introduces adversary-correction constraints to ensure the correct classification of adversarial data and minimizes changes to the model parameters. We propose an efficient cutting-plane-based algorithm to iteratively solve the large-scale nonconvex optimization problem by approximating the feasible region through polyhedral cuts and balancing between robustness and accuracy. Computational experiments on standard datasets such as MNIST and CIFAR10 demonstrate that the proposed approach significantly improves robustness, even with a very small set of adversarial data, while maintaining minimal impact on accuracy. 
    
\end{abstract}
\fi

\section{Introduction}
There has been remarkable progress in deep learning in the last two decades \cite{lecun2015deep}, and state-of-the-art deep learning models are even on-pair with human performance for specific image classification tasks. Still, they can be very sensitive towards adversarial attacks, \textit{i.e.}, tiny carefully selected perturbations to the inputs \cite{machado2021adversarial,goodfellow2014explaining,szegedy2014intriguing,kurakin2016adversarial}. This sensitivity is well known and is a major safety concern when incorporating deep learning models in autonomous systems \cite{huang2018safety,papernot2017practical}. For example, \cite{eykholt2018robust} showed that a classifier of traffic signs could be fooled by adding small perturbations resembling graffiti to the signs. Consequently, verifying whether models can be fooled by adversarial perturbations and developing techniques to train robust models have become active research areas in recent years \cite{ruan2018reachability, katz2017reluplex, botoeva2020efficient, goodfellow2014explaining, madry2017towards}.
Training a deep learning model with good performance on training data that is also robust against adversarial perturbations is an intriguing topic from a mathematical programming (MP) perspective. Training a deep learning model, such as a deep neural network (DNN), on a given training data set alone is challenging. DNNs used for practical applications can easily have more than a million trainable parameters, \textit{i.e.}, variables. On top of that, the loss function is typically highly nonconvex. The sheer size of the problems typically renders higher-order methods ill-suited, and the nonconvexity, in combination with the size, renders deterministic global optimization algorithms computationally intractable. When considering the goal of guaranteeing the robustness of a model against adversarial perturbations during the training, we are left with a truly challenging optimization problem.

Here, we focus on improving the robustness of DNNs in an image classification setting. We consider the task of fine-tuning a pre-trained model to improve robustness by utilizing a small given set of adversarial examples. Some of the deterministic approaches for finding adversarial examples build upon solving a mixed-integer linear programming (MILP) problem to find each example or show that none exists \cite{botoeva2020efficient,fischetti2018deep}. Due to the nonconvexity or the large size of the resulting MILP, adversarial examples can be difficult or time-consuming to find. Especially when the DNN is somewhat more robust, it becomes computationally more challenging to find adversarial examples. For this reason, the goal is to improve robustness using only a small set of adversarial examples.

We propose a nonlinear programming (NLP) formulation for fine-tuning a pre-trained DNN and adjusting its parameters to improve classification accuracy on adversarial data while keeping a good accuracy on the initial training data. The formulation builds on adding constraints that the adversarial examples must be classified correctly, that the loss function on training data should be kept at a certain level, and that the change to the DNN should be minimal. We prove that an optimal solution to this finite fine-tuning problem gives a robust DNN. However, the resulting problem is huge, from a classical MP perspective, with a large number of highly nonlinear and nonconvex constraints. Directly optimizing this fine-tuning NLP problem is not computationally tractable for DNN models and data sets of relevant size. We present an algorithm for approximately solving the fine-tuning problem. The idea is to iteratively solve a sequence of subproblems that approximate the feasible set by a polyhedral approximation. This is used together with a line search and a multiobjective selection strategy to balance between accuracy on training- and adversarial data. Computational experiments show that our approach can significantly improve robustness -- with minimal change in accuracy on clean data -- using a small set of adversarial examples.


Verifying that a DNN is robust against adversarial perturbations is, on its own,  $\mathcal{NP}$-complete \cite{katz2017reluplex}. Instead, the robustness and sensitivity of the DNN are estimated by considering the success rate of an adversarial attacker. Here, the adversarial attacker refers to an algorithm designed to make small changes to the input data to make the DNN classify the data incorrectly.

The paper is structured as follows. In Section~1.1 briefly summarizes some of the earlier work on improving the robustness of DNNs. For clarity, we briefly discuss our notation and some terminology in Section~\ref{sec:notation}. In Section~\ref{sec:formulation}, we present the adversary-correction problem, and our fine-tuning algorithm is presented in Section~\ref{sec:alg}. Some implementation details are discussed in Section~\ref{sec:impl}, and the computational experiments are presented in Section~\ref{sec:experiments}.

\subsection{Related work}

This section summarizes the recent contributions and outcomes regarding the robustness and safety of neural networks. There has been growing attention on the topic across different fields of research.

Processing on the dataset is one way to cope with adversarial attacks on the neural networks;
Pre-processing on the test dataset for de-noising can help detect and remove adversarial attacks before the test evaluation~\cite{liao2018defense}; \cite{liu2021training} proposed the adversarial training propagation algorithm that injects noise into the hidden layers.

Pruning-based methods have shown remarkable performance in reducing over-fitting \cite{lecun1989optimal,hassibi1993optimal,yu2022combinatorial}, which also leads to improvement in robustness; \cite{sehwag2020hydra} introduced a pruning-based method against adversarial attacks for large neural network size (often millions of parameters); \cite{zhao2023model} introduced a  MILP-based feature selection method that achieves improved robustness while maintaining accuracy; \cite{tramer2017ensemble} proposed an ensemble adversarial training approach that utilizes a decreasing gradient masking during the training process.

A common approach to train more robust models is to add adversarial data into the training set. For example, \cite{kurakin2016adversarial, szegedy2014intriguing} have shown that feeding adversarial data to the training set during the training process can result in a form of regularization and improve adversarial robustness. These approaches often rely on large sets of adversarial data and are therefore not relevant in our setting, where we assume we only have access to a small set of adversarial data.  We also tested the simplest approach of directly adding adversarial data to the training set in the numerical experiment section, and from the results, it was clear that larger adversarial data sets are needed for this approach to be effective.  

Many research studies have focused on adversarial training (AT), which aims to minimize the loss function of adversarial attacks with the maximum distortions. Various gradient-based methods have been explored to generate an attacker with maximum distortion under this framework, \eg, Fast Gradient Sign Method \cite{goodfellow2014explaining},  Projected gradient descent method (PGD) \cite{madry2017towards}, and  Frank-Wolfe (FW) Method~\cite{tsiligkaridis2022understanding}. 
PGD is regarded as the main algorithm for adversarial training with the power of finding maximum distortion, however, at the cost of computational expense. 
Compared to PGD's high computational cost, FW uses fewer iterations to find near-maximal distortion.

Instead of generating adversarial distortion for each data point in the training dataset, \cite{qin2019adversarial} linearized the loss of adversarial attackers by including a local linearity regularization and \cite{raghunathan2018certified} used a semidefinite programming regularization as an adversarial certification. 

The verification of a neural network is closely connected to the robust training and provides inspiring insights for improving robustness against data poisoning (\eg, \cite{wu2021tightening, sosnin2024certified}); \cite{zugner2019certifiable} and  \cite{shi2021fast} have introduced certified training method using interval bound propagation with regularization to tighten the certified bounds.

Due to its modeling power, mixed-integer programming has gained more popularity in neural network training. The nonlinearity of the forward function of neural networks can be encoded by mixed-integer programming \cite{fischetti2018deep}. This technique has led to MIP-based training methods, especially for neural networks with simple architecture, such as the binary neural network \cite{Aspman_Korpas_Marecek_2024} and integer-valued neural networks \cite{thorbjarnarson2023optimal}. Furthermore, \cite{toro2019training} emphasized robustness and simplicity by formulating the objective function as min-weight and max-margin, respectively, and conducted a hybrid approach to balance these two principles.

\subsection{Notation}\label{sec:notation}
We briefly introduce some notations and terminology that are used throughout the paper. $w$ denote all the trainable parameters, \ie, weights and biases $\{w^l,b^l\}_{l=1,
\dots, L-1}$, of the neural network. $\mathcal{X}_{\text{train}}:=\{x_n,y_n\}_{n=1,\dots,N}$ denotes the clean training dataset with input $x_n$ and labels $y_n$. 

Our framework is not limited to a specific loss function, and we use 
$$\ell (w):= \ell (w;\{x_n,y_n\}_{n=1,\dots,N}),$$ 
to denote the loss function over the training dataset $\mathcal{X}_{\text{train}}$ for the neural network with parameters $w$. The forward function, \textit{i.e.}, the inference of a classifier with parameters $w$ is denoted by $C(\cdot;w)$ and defined as $$C(\cdot;w):= \argmax_{i\in I} f_i(\cdot;w),$$ where $I$ is the set of labels, and $f_i(\cdot;w)$ is the classification confidence of label $i$ by the neural network. $f_i(\cdot;w)$ is, thus, the function from the input to the $i$-th output of the neural network. We define an epsilon neighborhood around the input $x$ as 
$$ B_{\varepsilon}(x):=\{\tilde{x}\in\mathbb{R}^m\mid  \ \|x-\tilde{x}\|_{\infty} \leq \varepsilon\},$$
with $\varepsilon > 0$. We say that any point in $B_{\varepsilon}(x)$ is an $\varepsilon$ perturbation of $x$. As adversarial data plays a central role in the paper, we include a formal definition of what we mean by adversarial data.

\begin{definition}
 We say that the input $\tilde{x}$ is an adversarial data point if it is misclassified and $ \exists x \in \mathcal{X}_\text{train}: \ \tilde{x} \in B_{\varepsilon}(x)$ and $x$ is correctly classified by the neural network.
\end{definition}

We use the notation $\mathcal{X}_\text{adv}:=\{\tilde{x}_n, y_n\}_{n=1,\dots, M}$ to denote a set of adversarial data, \textit{i.e.}, a set of data points $\tilde{x}_n$ with $y_n$ as its true unperturbed label that is wrongfully classified due to a small perturbation.

\begin{definition}(Resilience set)
    We define the resilience set for data point $x$ and the neural network with parameters $w$ as a subset of $\mathcal{B}_{\varepsilon}(x)$ that includes all the perturbations of $x$ within $B_{\varepsilon}(x)$ resulting in the same classification as $x$. We denote 
    $B_{\varepsilon}(x, w) \subset B_{\varepsilon}(x)$  the resilience set, and is given by $B_{\varepsilon}(x, w):= \{\tilde{x} \in B_\varepsilon \mid C(\tilde{x},w) = C(x,w) \}$.
\end{definition}

The resilience set helps us analyze robustness. Suppose that a data point $x$ is correctly classified and  $B_{\varepsilon}(x, w) = B_{\varepsilon}(x)$, then we know that this data point cannot be misclassified by a $\varepsilon$ perturbation, and we can consider the neural network to be robust for this specific input.

\begin{definition}
    Let $\mu\left(B_{\varepsilon}(x, w)\right)$ denote the Lebesgue measure of  $B_{\varepsilon}(x, w)$. We define the robustness of data point $x$ for the given neural network as 
$$R(x):=\frac{\mu\left(B_{\varepsilon}(x,w)\right)}{\mu\left(B_{\varepsilon}(x)\right)}.$$
\end{definition}
Here, $R(x)=1$ implies that the data point $x$ cannot be misclassified by a $\varepsilon$ perturbation. But, more importantly, the value of $R(x)$ contains more information about robustness. $R(x)$ close to zero implies that the classification of the data point can easily be changed by a small perturbation. Thus, smaller $R(x)$ will increase the likelihood of an adversarial attacker being successful in changing the classification of $x$. With $R(x)$ close to zero, even a small random perturbation will likely change the classification of $x$.

\section{Robust training of Neural Networks}\label{sec:formulation}
The sensitivity that enables small, sometimes even invisible, perturbations to the input to change the classification is undesirable. The goal of robust training is to make the neural network less sensitive to such perturbations. First, we introduce a formal definition of what we mean by a robust neural network.

\begin{definition}[$\varepsilon$-robust]
Given a neural network  $\mathcal{X}_{\text{corr}} \subset \mathcal{X}_{\text{train}}$ is the subset of training data that are all classified correctly, we say that the network is $\varepsilon$-robust if $\nexists \ \tilde{x} \in B_{\varepsilon}(x)$ for any $x \in \mathcal{X}_{\text{corr}}$ such that $\tilde{x}$ is classified differently than $x$.
\end{definition}

According to Definition~4, the neural network is $\varepsilon$-robust if no $\varepsilon$ perturbation to the input can make a correctly classified point from the training set misclassified. The strongest form of $\varepsilon$-robustness would be that all data points in the training set $\mathcal{X}_{train}$ are classified correctly, and none of them can be misclassified by an $\varepsilon$ perturbation. However, achieving such a degree of robustness is unrealistic, as correctly classifying all training data points is typically not possible for real-world applications. 

We focus on fine-tuning a neural network to increase robustness; The ultimate goal of which would be to change the parameters of the neural network as little as possible while ensuring that the network is $\varepsilon$-robust. Given a trained neural network with parameters $\hat{w}$ and the set of correctly classified training points $\mathcal{X}_{\text{corr}}$, we can form the robust fine-tuning problem
\begin{equation}\label{eq:robust_training}
    \begin{aligned}
        \min_w~& \|w-\hat{w}\|_2^2 \\
        \st~& C(\tilde{x};w)= C(x;\hat{w}), \forall \tilde{x} \in B_{\varepsilon}(x), \forall x \in \mathcal{X}_{\text{corr}}.
    \end{aligned}
\end{equation}

The constraint in \eqref{eq:robust_training} states that the classification must be corrected for all the previously correctly classified training points in $\mathcal{X}_{\text{corr}}$, as well as any $\varepsilon$ perturbation of those points. It can be observed that the loss function $\ell (w)$ is not included in \eqref{eq:robust_training}, as its constraints ensure that the accuracy on $\mathcal{X}_{\text{train}}$ remains unchanged. If we could solve problem \eqref{eq:robust_training}, we would obtain a fine-tuned neural network that is $\varepsilon$-robust. Unfortunately, solving \eqref{eq:robust_training} is computationally intractable for neural networks and training sets of relevant size. Problem \eqref{eq:robust_training} may look innocent, but keep in mind that it includes an infinite number of constraints for each corresponding point in $\mathcal{X}_{\text{corr}}$. Furthermore, each of the constraints $C(\tilde{x};w)= C(x;\hat{w})$ are, in general, nonconvex, and $w$ is high dimensional (more than $10^5$ variables for the larger neural network we test with).

For a given neural network architecture, and depending on the data set $\mathcal{X}_\text{train}$, it might not be possible to obtain an $\varepsilon$-robust neural network by a fine-tuning process. Instead, we focus on improving robustness. For clarity, we include a formal definition of what we mean by a neural network being more robust.
\begin{definition}[Strictly more robust over $\mathcal{X}_{train}$]
Given two neural networks, a and b, with parameters $w_a$ and $w_b$ that both classify the same subset of training data $\mathcal{X}_\text{corr}\subset\mathcal{X}_\text{train}$ correctly. Then we say that neural network a is strictly more robust over $\mathcal{X}_{train}$ than neural network b if
$$R_a(x)=\frac{\mu\left(B_{\varepsilon}(x,w_a)\right)} {\mu\left(B_{\varepsilon}(x)\right)}\ \geq R_b(x)=\frac{\mu\left(B_{\varepsilon}(x,w_b)\right)} {\mu\left(B_{\varepsilon}(x)\right)} \quad \forall x \in \mathcal{X}_\text{corr} ,$$
and $\exists \ \bar{x}\in \mathcal{X}_\text{corr}:R_a(\bar{x}) > R_b(\bar{x})$.
\end{definition}

As mentioned, evaluating if a data point in $\mathcal{X}_\text{train}$ can be misclassified by a small perturbation is computationally challenging. Furthermore, evaluating $\mu\left(B_{\varepsilon}(x,w_b)\right)$, and even constructing set $B_{\varepsilon}(x,w_b)$, is computationally intractable. Therefore, we use the success rate of an adversarial attacker to evaluate if a neural network is more robust in the computational results.

To obtain a more tractable fine-tuning problem than \eqref{eq:robust_training}, we choose to focus on adjusting the neural network based on a given set of adversarial data $\mathcal{X}_{\text{adv}}$ to improve robustness. The adversarial data can, \textit{e.g.}, represent critical misclassifications made by the neural network. Finding such adversarial data can be computationally expensive, in fact $\mathcal{NP}$-complete \cite{katz2017reluplex}. Improving robustness with a small set of adversarial data is, therefore, desirable and the focus of our work. In practice, it is often easy to find adversarial data points for sensitive neural networks but increasingly difficult for more robust neural networks. Here, we have not focused on how to find adversarial examples; we simply assume a set of adversarial data $\mathcal{X}_{\text{adv}}$ is given. For more details on techniques for finding adversarial examples, we refer to \cite{katz2017reluplex,botoeva2020efficient,tsay2021partition}.

\subsection{The adversary-correction problem}

As mentioned, directly optimizing \eqref{eq:robust_training} is not computationally tractable as such. Therefore, we instead focus on fine-tuning to correct the misclassified adversarial data. We develop a process based on an adversary-correction problem reducing the constraints in \eqref{eq:robust_training} to a finite set of constraints defined by the adversarial data set $\mathcal{X}_\text{adv}:=\{\tilde{x}_n, y_n\}_{n=1,\dots, M}$, where $\tilde{x}_n$ is a perturbed data point from $\mathcal{X}_\text{train}$  and $y_n$ is its true unperturbed label. 

To enforce the correct classification of adversarial data, we could use the constraints 
\begin{equation}
    C(\tilde{x}_n;w)= y_n, \quad \forall (\tilde{x}_n, y_n) \in \mathcal{X}_\text{adv}.
\end{equation}
However, the forward pass function $C(\cdot;w)$ involves the $\argmax$ operator, which further complicates the constraint. Instead, we promote the correct classification by adding inequality constraints that the classification confidence in the correct label must be greater than or equal to confidence in all other labels. 

For a pair of classification labels $i \in I$ and $j \in I$, $i$ is considered to be at least as likely classification as $j$ for input $\tilde{x}$ if
\begin{equation}
    f_i(\tilde{x};w) \geq f_j(\tilde{x};w).
\end{equation}
Recall that  $f_i$ is a function representing the $i$-th output as a function of the neural network's input, and $f_i(\tilde{x};w)$ represents the classification confidence of the $i$-th label. Given an adversarial data point $\tilde{x}$ and its correct classification $y$, we can ensure that $y$ is the classification by the neural network by adding the constraints
\begin{equation}
\label{eq:corr_class}
        f_{y}(\tilde{x};w) \geq f_j(\tilde{x};w) + \delta, \quad \forall j \in I \setminus y,
\end{equation}
where $\delta > 0 $ (\textit{e.g.,} $ \delta=10^{-5}$) is a threshold constant to guarantee $f_y(\cdot)$ is strictly larger than the classification confidence of other labels.  
Note that the only variables here are $w$, as the adversarial data points and their labels are given in set $\mathcal{X}_\text{adv}$.

With multiple adversarial data points in $\mathcal{X}_\text{adv}$, we can form the constraints in \eqref{eq:corr_class} for each data point. To simplify the notation, we introduce the $\mathcal{G}( \mathcal{X}_\text{adv})$ that defines the feasible set of constraints \eqref{eq:corr_class} generated for each data point. The feasible set is, thus, defined as
\begin{equation} 
\label{eq:adv_cons_nonlin}
    \mathcal{G}( \mathcal{X}_\text{adv}) := \{ w \in \mathbb{R}^J \mid f_{y}(\tilde{x};w) \geq f_{i}(\tilde{x};w) +\delta, \forall i\neq y, \forall (\tilde{x},y) \in \mathcal{X}_\text{adv}\},
\end{equation}
where $J$ is the total number of parameters of the neural network.

Projecting the neural network's parameters $\hat{w}$ onto $\mathcal{G}(\mathcal{X}_\text{adv})$ ensures that each adversarial data point can be correctly classified. However, as we are now only focusing on a potentially small set of adversarial data points, this could be detrimental to the overall classification accuracy. Therefore, we should also consider the loss function for the training data. We now define the adversarial correction problem as the projection of the neural network's parameters $\hat{w}$  onto the intersection of $\mathcal{G}( \mathcal{X}_\text{adv})$ and solutions with a no-worse loss function over the training data $\mathcal{X}_\text{train}$. The adversarial correction problem can then be formulated as 
\begin{subequations}\label{eq:robust-project_org}
\begin{align}
     \min~&\|w-\hat{w}\|_2^2\\
        \st~&
        w \in \mathcal{G}( \mathcal{X}_\text{adv}),\label{eq:adv_org} \\
        & \ell (w;\mathcal{X}_\text{train}) \leq  \ell(\hat{w},\mathcal{X}_\text{train}) \label{eq:loss_org}.
\end{align}
\end{subequations}
In practice, it can be beneficial to relax constraint \eqref{eq:loss_org} to $\ell (w;\mathcal{X}_\text{train}) \leq  \ell(\hat{w},\mathcal{X}_\text{train}) + \xi$ to allow a trade-off between training accuracy and classification of the adversarial data. Especially if the adversarial data set is larger, there might not exist a solution to \eqref{eq:robust-project_org} without relaxing the constraint on the loss function. We will come back to this trade-off later in the next section.

Next, we provide a theoretical justification for the adversarial correction problem \eqref{eq:robust-project_org} and prove that an optimal solution can ensure $\varepsilon$-robustness. In the analysis, we build on the assumption that the output functions $f_i$ are Lipschitz continuous and where $L < \infty$ is the largest Lipschitz constant of the output functions. As all the parameters of the neural network are bounded (in practice bounded to reasonably small values, \eg, $\pm 10^6$), the assumption that $f_i$ are Lipschitz continuous is reasonable and not restrictive in practice.
\begin{theorem}
There exists a finite set of adversarial data $\mathcal{X}_\text{adv}$, such that an optimal solution (if one exists) to the adversarial correction problem \eqref{eq:robust-project_org} is also a feasible solution to the robust fine-tuning problem \eqref{eq:robust_training}.
\end{theorem}\label{thm:adv_corr_cover}
\begin{proof}
For a neural network whose weights are given by an optimal solution to \eqref{eq:robust-project_org}, there exists a small neighborhood, given by a ball with radius $r \geq \frac{\delta}{L} >0$, around each adversarial data point $\tilde{x} \in \mathcal{X}_\text{adv}$ in which the classification remain correct. This follows from the fact that the input-output mapping of the neural network is Lipschitz continuous, and the output corresponding to the correct classification is $\delta$ greater than the other outputs. As $r >0$, we can make sure that a finite set of correct classification neighborhoods cover the sets $B_{\varepsilon}(x)\ \forall x \in \mathcal{X}_\text{corr}$, thus ensuring $\varepsilon$-robustness.
\end{proof}
Theorem~1 shows that the infinite number of constraints in robust fine-tuning problem \eqref{eq:robust_training} is not necessary to obtain $\varepsilon$-robustness. This serves as an important theoretical motivation for focusing on solving the finite adversarial correction problem \eqref{eq:robust-project_org} and shows that you do not have to use an infinite number of constraints nor integrate an adversarial attacker into the training problem (resulting in a bilevel problem) to ensure robustness. However, the number of adversarial data points needed to ensure $\varepsilon$-robustness can still be huge, and computationally intractable. However, to improve robustness, we clearly need fewer adversarial data points. To support this claim, consider the proof of Theorem~1. If the neural network is far from robust, \ie, $R(x)$ is close to zero for many points $x$ in $\mathcal{X}_\text{corr}$, then we only need to cover a small part of each set $B_{\varepsilon}(x)$ with adversarial data and the resulting correctly classified neighborhoods to improve robustness. From the theory, we may still need many adversarial data points to improve robustness, but it suggests that the number of points needed can be manageable. Our main hypothesis behind this work has been: \say{\textit{Robustness of a neural network can be improved by using relatively few adversarial examples in the adversarial correction problem}}. The hypothesis is supported by the numerical results, which show that even one adversarial data point per classification label can improve robustness.

Even if the adversarial correction problem~\eqref{eq:robust-project_org} is far simpler than the robust fine-tuning problem~\eqref{eq:robust_training}, it is still far from easy to solve. Both the constraints \eqref{eq:adv_org} and \eqref{eq:loss_org} are nonconvex, and, for relevant applications, we are typically dealing with more than $10^4$ variables. Therefore, finding a globally optimal solution to the adversarial correction problem~\eqref{eq:robust-project_org} still seems computationally intractable. In fact, optimality is not the main target as the objective function in \eqref{eq:robust-project_org} is not directly related to the neural network's performance. Rather, it is intended to help the algorithm by making the search more targeted. Feasibility is instead our main interest, as the constraints in \eqref{eq:robust-project_org} enforce the performance criteria that we are after. In the next section, we present an algorithm for approximately solving \eqref{eq:robust-project_org} to improve robustness. 

\section{The cutting-planed based fine-turning algorithm}\label{sec:alg}

The difficulty of solving \eqref{eq:robust-project_org} boils down to the onlinearity and nonconvexity of the constraints. Due to the problem size, directly applying second-order methods, such as interior point methods or sequential quadratic programming, does not seem computationally tractable. As a motivation for this statement, the larger neural network we consider has approximately 180,000 parameters, and the Hessian of the Lagrangian would be dense due to the neural network structure. A first-order method, therefore, seems better suited for the task and is also the typical choice for training neural networks. 

We will use a so-called cutting plane algorithm, where we approximate each nonlinear constraint with linear constraints. To linearize the constraints, we apply what is often referred to as gradient cuts, and we briefly describe their derivation. Recall, a first-order Taylor expansion to approximate a function $g:\mathbb{R}^n \rightarrow \mathbb{R}$ at the point $x_0$ is given by
\begin{equation*}
        \bar{g}(x):= g(x_0) + (x-x_0)^  \top \nabla_x g(x_0),
\end{equation*}
where $\nabla_x g$ denotes a (sub)gradient of $g$ with respect to $x$. By setting $\bar{g}(x)\leq 0$, we get a linearized constraint, or a cut, of the constraint $g(x)\leq 0$.

We now apply this linearization technique to the constraints $f_{y}(\tilde{x};w) \geq f_{i}(\tilde{x};w)$ that enforce the correct classification of adversarial data point $\tilde{x}$. The resulting cut from one of these constraints is
$$f_y(\tilde{x};\hat{w}) + (w-\hat{w})^\top \nabla_{w} f_y(\tilde{x};\hat{w}) - f_j(\tilde{x};\hat{w}) - (w-\hat{w})^\top \nabla_{w} f_j(\tilde{x};\hat{w})\geq \delta,$$
where $\nabla_{w} f_i(\tilde{x};\hat{w})$ is a gradient of the $i$-th output of the neural network with respect to the neural network's parameters $w$. Note that, depending on the type of activation functions used, the function $f_i$ may not be smooth everywhere, \textit{i.e.}, not necessarily continuously differentiable. At the nonsmooth points, we instead use a subgradient, which is also done in stochastic gradient descent for training neural networks.

Applying this linearization technique to all constraints in \eqref{eq:adv_cons_nonlin} gives us a linear approximation of $\mathcal{G}(\mathcal{X}_\text{adv})$, which is given by
\begin{equation}
\label{eq_lin-set1}
\begin{aligned}
     \bar{\mathcal{G}}( \mathcal{X}_\text{adv}):= & 
        \{ w \in \mathbb{R}^J\mid f_y(\tilde{x};\hat{w})\\
        & + (w-\hat{w})^\top \nabla_{w} f_y(\tilde{x};\hat{w}) - f_i(\tilde{x};\hat{w}) - (w-\hat{w})^\top \nabla_{w} f_i(\tilde{x};\hat{w})\geq \delta, \\
        &~
         \forall i\in I \setminus y, \forall (\tilde{x},y) \in \mathcal{X}_\text{adv} \}.
\end{aligned}
\end{equation}
Similarly, we linearize the loss function and add the constraint that the linearized loss function should not increase, \ie,
\begin{equation}
    \label{eq:lin-loss}
     (w-\hat{w})^\top\nabla_w \ell(\hat{w})\leq 0.
\end{equation}
If the loss function is smooth, then \eqref{eq:lin-loss} restricts us from moving in an ascent direction from the current solution $\hat{w}$.

Using the linear approximations \eqref{eq_lin-set1} and \eqref{eq:lin-loss} of the constraints in problem \eqref{eq:robust-project_org}, we form the approximated adversarial-correction problem as the convex quadratic programming (QP) problem
     \begin{equation}\label{eq:adv_qp_loss}
        \begin{aligned}
            P(\hat{w},\mathcal{X}_\text{adv}):=\argmin~&\|w-\hat{w}\|_2^2\\
        \st~&
        w \in \bar{\mathcal{G}}( \mathcal{X}_\text{adv}), \\
        &(w-\hat{w})^\top\nabla_w \ell(w)\leq 0.
        \end{aligned}
    \end{equation}

As \eqref{eq:adv_qp_loss} is an approximation, we need a measure to evaluate the quality of solution $w$. For fine-tuning, we mainly care about the feasibility, or the violation, of the nonlinear constraints \eqref{eq:adv_org} and \eqref{eq:loss_org}. We define the total violation of the constraints in \eqref{eq:adv_org} as 
\begin{equation}
\label{eq:total_con}
\begin{aligned}
    V(w;\mathcal{X}_\text{adv}): =   \sum_{ \forall (\tilde{x},y) \in \mathcal{X}_\text{adv}} |\min_{i\neq y}  \{f_y(\tilde{x};w) - f_i(\tilde{x};w),0\} |,
\end{aligned}
\end{equation}
where $V(w;\mathcal{X}_\text{adv}) = 0$ indicates that the predicted labels of all adversarial data points have been corrected. Generally, a lower constraint violation $V(w;\mathcal{X}_\text{adv})$ is considered better, as it implies that the confidence in the misclassification is reduced. For evaluating the quality of the solution concerning constraint \eqref{eq:loss_org}, we use $\ell(w;\mathcal{X}_\text{train})$ as the metric.

There are two important aspects of problem \eqref{eq:adv_qp_loss} that we need to discuss. First, the linear approximations in \eqref{eq:adv_qp_loss} are only accurate close to the initial solution $\hat{w}$. In fact, the minimizer of \eqref{eq:adv_qp_loss} can be a worse solution to \eqref{eq:robust-project_org} compared to the initial solution. This can be dealt with by generating additional linearizations of the constraints, ideally building a better linear model. The approach for updating the linear model is discussed in the next section. Secondly, it is important to remember that the nonlinear constraints \eqref{eq:adv_org} and \eqref{eq:loss_org} are both nonconvex. As a consequence, $\bar{\mathcal{G}}( \mathcal{X}_\text{adv})$ and $(w-\hat{w})^\top\nabla_w \ell(w)\leq 0$ do not necessarily form an outer approximation of the feasible set. Still, the linearized constraints contain local information about the constraints and guide us towards regions of interest. Due to the nonconvexity, the convergence of our approach is not guaranteed. However, due to the size of problem \eqref{eq:robust-project_org}, the authors are not aware of any method that would be computationally tractable and guarantee optimality or feasibility of problem \eqref{eq:robust-project_org} for neural networks and data sets of relevant size.

The realistic goal of our algorithm is not to find an optimal solution of the adversarial correction problem \eqref{eq:robust-project_org}, nor necessarily a solution that satisfies all constraints, but to obtain a solution that is clearly better than the initial solution $\hat{w}$. By better, we mean in terms of the total constraint violation \eqref{eq:total_con} and in the loss function on the training data. We elaborate on how these two objectives are considered in Section~\ref{sec:bi-obj}.

\subsection{Updating the linear approximation}
By linearizing each constraint only at the initial point $\hat{w}$, we do not gain any insight into the curvature of the constraints, and the approximation only captures its behavior in the vicinity of $\hat{w}$. The minimizer of \eqref{eq:adv_qp_loss} is, therefore, usually a poor solution to \eqref{eq:robust-project_org}. To accumulate more information about the constraints, we expand the QP model by linearizing the constraints at the iterative solutions obtained by our algorithm similar to a Bundle method \cite{makela2002survey}, or Kelley's cutting plane method \cite{kelley1960cutting}. 

Assume we have performed $k$ iterations resulting in the trial, or potential, solutions $w^1, w^2, \dots, w^k$. From there, we form a set $\mathcal{W}_\text{cut}:= \{w^1, w^2, \dots, w^k\}$ that we will use to build our polyhedral approximation of the constraints. We can then form the approximation of the feasible set $ \mathcal{G}(\mathcal{X}_\text{adv})$ in \eqref{eq:adv_cons_nonlin} as
\begin{equation}\label{eq:qp_cut}
\begin{aligned}
\tilde{\mathcal{G}}(\mathcal{X}_\text{adv},\mathcal{W}_\text{cut}):= 
&\{ w\in \mathbb{R}^J \mid  f_i(\tilde{x};w^k)   - f_{y}(\tilde{x};w^k) \\
&+ (w-w^k)^\top (\nabla_w f_i(\tilde{x};{w^k})-\nabla_w f_{y}(\tilde{x};{w^k}))  +\delta \leq 0,\\
 &~\forall i\in I \setminus y,\ \forall w^k \in \mathcal{W}_\text{cut}, \ \forall (\tilde{x},y) \in \mathcal{X}_\text{adv}\}.
\end{aligned}
\end{equation} 
Similarly, we form a polyhedral approximation of the feasible set of constraint \eqref{eq:loss_org} by 
\begin{equation}
    \mathcal{H}(\mathcal{W}_\text{cut}):= \{w \in \mathbb{R}^J \mid \ell (w^k) - \ell(\hat{w})  + (w-w^k)^\top 
    \nabla_w \ell ({w^k})  \leq 0,\quad \forall w^k \in  \mathcal{W}_\text{cut}\}.
\end{equation}

As $\tilde{\mathcal{G}}(\mathcal{X}_\text{adv},\mathcal{W}_\text{cut})$ and $\mathcal{H}(\mathcal{W}_\text{cut})$ linearize the constraints at multiple points they can capture more details of the nonlinear constraints. We use $\tilde{\mathcal{G}}(\mathcal{X}_\text{adv},\mathcal{W}_\text{cut})$ and $\mathcal{H}(\mathcal{W}_\text{cut})$ to strengthen the approximated adversarial-correction problem \eqref{eq:adv_qp_loss} resulting in 
\begin{equation}\label{eq:qp_cut_iter}
\begin{aligned}
    P(\hat{w},\mathcal{X}_\text{adv}, \mathcal{W}_\text{cut}):=\argmin~&\|w-\hat{w}\|_2^2\\
        \st~&
        w \in \bar{\mathcal{G}}( \mathcal{X}_\text{adv}), \\
        & (w-\hat{w})^\top\nabla_w \ell(w)\leq 0,\\
        & w \in \tilde{\mathcal{G}}(\mathcal{X}_\text{adv},\mathcal{W}_\text{cut}),\\
        & w \in \mathcal{H}(\mathcal{W}_\text{cut}).
        \end{aligned}
\end{equation}

Note that \eqref{eq:qp_cut_iter} is a convex QP that typically can be solved efficiently. For large neural networks, the convex QP can still be computationally expensive. In Section~\ref{sec:impl}, we propose a simple algorithm to obtain an approximate solution to \eqref{eq:qp_cut_iter} by a block-coordinate-based approach to speed up the algorithm for large problems.  

One of the main components of our adversary-correction algorithm is iteratively solving \eqref{eq:qp_cut_iter}. After each iteration, the minimizer is added to the set $\mathcal{W}_\text{cut}$. Repeatedly solving \eqref{eq:qp_cut_iter}, with an increased number of cuts, gives us candidate solutions that typically improve in each iteration. The minimizer of \eqref{eq:qp_cut_iter} also gives us promising search directions.

We do not perform a typical line search, where we would search for the optimal step length in the given direction. Instead, we want to generate a set of alternative solutions. Denote the minimizer of \eqref{eq:qp_cut_iter} in iteration $k$ by $w^k$, then we can consider $d^k=(w^k - \hat{w})$ a search direction from the initial solution $\hat{w}$. It is well known that first-order methods can take too long steps, especially in early iterations. This is our motivation for further exploring the search space in the direction $d^k$ from the initial solution $\hat{w}$. We can then define a set of alternative solutions in iteration $k$ as 
$$\left\{\alpha w^k + (1-\alpha)\hat{w} \mid \  \alpha = 0.1, 0.2, \dots, 1 \right\}.$$
We could perform a more sophisticated line search, but the main goal here is to generate alternative solutions. A classical line search would also require a measure to combine the constraint violations and loss function, and we want to avoid combining them at this stage.

\subsubsection{Considering small changes in the adversarial data}

To further improve the robustness of a neural network, it would be desirable to not only focus on the data points in $\mathcal{X}_\text{adv}$ but also to restrict any small perturbation of these to be correctly classified. For example, we could consider an $\bar{\varepsilon}$ neighborhood around an adversarial data point $\tilde{x} \in \mathcal{X}_{adv}$ defined by the $\ell_1$-norm ball:
\begin{equation}
    B_{\bar{\varepsilon}}(\tilde{x}) := \{x\mid \|\tilde{x}-x\|_{1} \leq \bar{\varepsilon} \}=\{x \mid \sum_k |\tilde{x}_k -x_k | \leq \bar{\varepsilon} \},
\end{equation}
where $\bar{\varepsilon} >0$ is the radius. We cannot directly include all points from this $\bar{\varepsilon}$ neighborhood, but we can consider it in the approximated adversary-correction problem through a different linearization.

Consider the Taylor expansion of $f_i(x,w)$ at $(\tilde{x} ,\hat{w})$:
\begin{equation}\label{eq:taylor_w_x_adv}
    f_i(x,w) \approx f_i(\tilde{x},\hat{w}) + (w-\hat{w})^\top \nabla_w f_i (\tilde{x},\hat{w}) + (x-\tilde{x})^\top \nabla_x f_i (\tilde{x},\hat{w}),
\end{equation}
where $\nabla_x f_i (\tilde{x},\hat{w})$ is a (sub)gradient of $f_i$ with respect to the neural network input $x$ at the point $(\tilde{x},\hat{w})$. Note that \eqref{eq:taylor_w_x_adv} is a linear approximation for how the $i$-th output of the neural network changes with its parameters and with the input.
As before, we can use this linear approximation in the constraints to force the correct classification of the adversarial data point $\tilde{x}$. Consider the requirement that the classification confidence in label $i$ should be greater than  the confidence in label $j$: we get the linearized constraint
\begin{equation}
\label{eq:lin-rob}
\begin{aligned}
    f_j(\tilde{x},\hat{w})-f_i(\tilde{x},\hat{w}) + (w-\hat{w})^\top (\nabla_w f_j(\tilde{x},\hat{w})- \nabla_w f_i(\tilde{x}, \hat{w})) + \delta \\
   \leq - (x-\tilde{x})^\top (\nabla_x f_j(\tilde{x}, \hat{w})- \nabla_x f_i(\tilde{x}, \hat{w})).
\end{aligned}
\end{equation}
By setting $x=\tilde{x}$, this reverts back to the linearized constraints in \eqref{eq_lin-set1}. Instead, we want \eqref{eq:lin-rob} to hold for any $x\in B_{\bar{\varepsilon}}(\tilde{x})$. This can be achieved by using the robust formulation
\begin{equation}\label{eq:robust_x}
    \begin{aligned}
   &f_j(\tilde{x},\hat{w})-f_i(\tilde{x},\hat{w}) + (w-\hat{w})^\top (\nabla_w f_j(\tilde{x},\hat{w}) - \nabla_w f_j (\tilde{x},\hat{w})) + \delta \\
   &\leq - \max_{\|x-\tilde{x}\|_{1}\leq \bar{\varepsilon}}\{(x-\tilde{x})^\top (\nabla_x f_j(\tilde{x},\hat{w})- \nabla_x f_i(\tilde{x},\hat{w})) \}.
\end{aligned}
\end{equation}
As the maximum of a linear function over the $\ell_1$ ball is obtained at one of the vertexes, the constraint above simplifies into 
\begin{equation}\label{eq:robust_x2}
    \begin{aligned}
 &f_j(\tilde{x},\hat{w})-f_i(\tilde{x},\hat{w}) + (w-\hat{w})^\top (\nabla_w f_j(\tilde{x},\hat{w}) - \nabla_w f_j (\tilde{x},\hat{w})) + \delta
    \\
&\leq - \bar{\varepsilon} \cdot \max_m |{(\nabla_x f^m_{j}(\tilde{x},\hat{w})- \nabla_x f^m_{i}(\tilde{x},\hat{w}))}|,
\end{aligned}
\end{equation}

where $m$ denotes the index of $\nabla_x f(\cdot)$. Thus, the result is a strengthening of the original linearized constraints by subtracting the constant $\bar{\varepsilon} \cdot \max_m |{(\nabla_x f^m_{j}(\tilde{x},\hat{w})- \nabla_x f^m_{i}(\tilde{x},\hat{w}))}|$ from the right-hand side.

We can use this linearization approach for all the constraints in problem \eqref{eq:qp_cut_iter}. Ideally, this should account for some small perturbations on the adversarial data points. We noticed from the computational experiments that this approach seems to more quickly reduce the  $V(w;\mathcal{X}_\text{adv})$, but it did not significantly change the final result. Some results showing the impact of \say{strengthened} linearization are presented in Section~5. But, we chose to leave this strategy out of the main numerical experiments as it introduces an additional parameter $\bar{\varepsilon}$ without any clear benefits.

\subsection{Balancing between the adversary-correction and loss on training data}
\label{sec:bi-obj}
In practice, it is not clear whether the misclassification of all adversarial data points can be successfully corrected, and it is less likely that this can be achieved without sacrificing some performance on the training data. Therefore, we need to somehow balance between violation of $\mathcal{G}(\mathcal{X}_\text{adv})$ and the no-worse loss constraint (\ie, \eqref{eq:adv_org} and \eqref{eq:loss_org}). To select between the two, we will use some basic concepts from multiobjective optimization \cite{eichfelder2021twenty}. We first introduce the following definitions.

\begin{definition}[Pareto optimal]
We say that a solution $w$ is Pareto optimal if it is not dominated by any other solution, \textit{i.e.}, there is no other solution that is equally good in all objective functions and strictly better in one objective.   
\end{definition}

\begin{definition}[Efficient front / Pareto front]
  The efficient front consists of points that are all Pareto optimal.
\end{definition}

We have two main objectives when evaluating the candidate solutions: the violation of $\mathcal{G}(\mathcal{X}_{adv})$, \ie, $V(w;\mathcal{X}_\text{adv})$, and the loss of the training data, \ie, $\ell(w)$. From the set of candidate solutions $\mathcal{W}_t$, we can evaluate which of them are Pareto optimal in the two objectives over the set $\mathcal{W}_t$. Note that these are only Pareto optimal among the candidate solutions $\mathcal{W}_t$, and better solutions might exist if we consider all possible parameters of the neural network. Still, Pareto optimality gives us a simple filter for discarding some of the candidate solutions.

We use the weighted sum technique to select the final solution from candidate solutions. For a given weight $0<\omega<1$, the final solution is selected as the minimizer of 
\begin{equation}\label{eq:weighted_sum}
\begin{aligned}
             \min~&\omega \tilde{\ell}(w) + (1-\omega) \tilde{V}(w,\mathcal{X}_\text{adv}),\\
     \st~& w \in \mathcal{W}_{t},
\end{aligned}
\end{equation}
where  $\tilde{\ell}$ and $\tilde{V}$ denote the scaled values of  $\ell(w)$ and $V(w,\mathcal{X}_\text{adv})$  between 0 and 1, and $\mathcal{W}_t$ is the set of all candidate solutions. We scale both metrics as it makes the selection of the $\omega$ parameter more intuitive. From the computational results in Section~5, the algorithm seems quite robust with regards to the choice of $\omega$. As expected, a different choice of $\omega$ gave the best result for the different test problems, but the different $\omega$ values tested all gave reasonable final results. Filtering the candidate solutions based on Pareto optimality seems to be most important, as it excludes poor solutions.

To better illustrate this idea, we present the candidate solutions --  what we refer to as the solution pool -- of a CNN model after 5 iterations with Algorithm~\ref{alg:model_update} in Figure~\ref{fig:efficient_front}. The y-axis present the loss on $\mathcal{X}_{\text{train}}$, \ie, $\ell(w)$ for each solution, and the x-axis present the \texttt{vio\_abs\_sum}, \ie, $V(w,\mathcal{X}_\text{adv})$.
The efficient front is colored in orange, and the point representing the pre-trained model is colored in purple. Comparing the initial point with the efficient front, it shows that Algorithm~\ref{alg:model_update} can achieve an improvement in $V(w,\mathcal{X}_\text{adv})$ with a very little increase in the loss function $\ell(w)$ on the training dataset. Note that these results are obtained after only 5 iterations, and we present these results only to keep the figure clearer as it contains fewer points. Typically, we obtain significantly better results with more iterations. 
\begin{figure}
\centering
\includegraphics[width=0.7\textwidth]{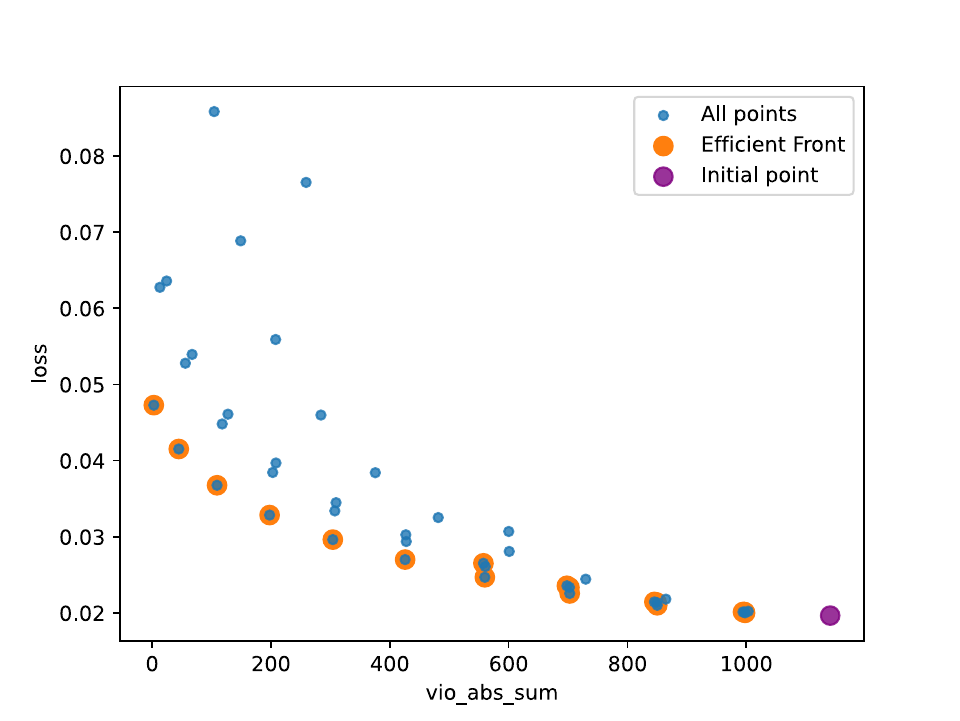}
\caption{Candidate solutions and their efficient front after 5 iterations with Algorithm~\ref{alg:model_update}, on the CNN model that is presented in Section~\ref{sec:experiments}. Note that we only ran 5 iterations to reduce the number of points for illustration purposes.}\label{fig:efficient_front}
\end{figure}

\subsection{Main algorithm}
We have now covered all components of our algorithm for improving robustness and finding an improved solution to \eqref{eq:robust-project_org}. The algorithm is presented as pseudo-code in Algorithm~\ref{alg:model_update}. In this algorithm, we use a maximum number of iterations $M$ as the stopping criteria. This is motivated by the fact that we are searching for an approximate solution to \eqref{eq:robust-project_org}, and limiting the number of iterations --and by extension the computational burden -- seems reasonable. Increasing the number of iterations generally produces better results, although the main improvements were often achieved within the first 15-20 iterations.
\begin{algorithm}[!htbp]
    \caption{The Cutting-plane based Adversary-Correction Algorithm}\label{alg:model_update}
\KwData{Pre-trained model parameters $\hat{w}$, $0<\omega <1$,
the maximum iterations number $M$, an adversarial dataset $\mathcal{X}_\text{adv}$\;}
\KwOut{$w$\;}

$\mathcal{W}_t \leftarrow \{w^0\} $  \hfill \# Initialize the solutions pool\;  
$\mathcal{W}_\text{cut} \leftarrow \emptyset$ \hfill \# Initialize the linear cuts set\;
\For{$k=1,\dots,M$}
    {$w^{k} \leftarrow P(w^0,\mathcal{X}_\text{adv},\mathcal{W}_\text{cut})$  \hfill\# solving \eqref{eq:qp_cut_iter}\;
   
     $\mathcal{W}_\text{cut} \leftarrow \mathcal{W}_\text{cut} \cup  w^{k}$\;
     $ \mathcal{W}_t\leftarrow \mathcal{W}_t \cup \{ \alpha w^{k} + (1-\alpha) w^0 \mid \forall  \alpha = 0.1,\dots, 1\} $

     \# Line search to expand the solutions pool\;
    }

    $w\leftarrow \argmin_{w \in  \mathcal{W}_t}\{\omega \tilde{\ell}(w) + (1-\omega) \tilde{V}(w,\mathcal{X}_{adv})  \}$\;

\end{algorithm}
In the next section, we describe some implementation details and how we can find an approximate solution to \eqref{eq:qp_cut_iter} in order to speed up the computations for larger neural networks.
\section{Implementation details}
\label{sec:impl}

The whole pipeline of the numerical experiment, \eg, the data preparation process, neural network training, and optimization modeling, are implemented on a MacBook Pro with an Apple M2 Pro CPU, 10 cores, and 32 GB memory. We trained neural networks using \textsc{PyTorch 2.2.1} \cite{paszke2017automatic}, and the optimization problems were solved using \textsc{Gurobi 11.0.1} \cite{gurobi}. The details of model architecture and data generation can be found at \hyperlink{https://github.com/shudianzhao/Adversary-correction}{https://github.com/shudianzhao/Adversary-correction}.

For larger neural networks, with over $10^5$ parameters, and with larger sets $\mathcal{X}_\text{adv}$ and $\mathcal{W}_\text{cut}$, it can become computationally costly to solve \eqref{eq:qp_cut_iter}. 
To tackle the increasing computational complexity and speed up the computations, \eqref{eq:qp_cut_iter} is solved by fixing a subset of variables and iterating in a block coordinate descent fashion.

An important realization is that we do not need an optimal solution to the convex QP problem \eqref{eq:qp_cut_iter}. Instead, we need to obtain a feasible solution -- ideally with an objective value close to the optimum. For this purpose, we propose a simple block coordinate optimization algorithm. The algorithm is designed for problems of the type
\begin{equation}
\label{eq:QP-fixed}
\begin{aligned}
    \min~&\|w-\hat{w}\|_2^2 \\
    \textrm{s.t.}~& w \in \mathcal{L}\subset \mathbb{R}^N,
\end{aligned}
\end{equation}
where $\mathcal{L}$ is defined by linear constraints. We start by initializing all variables as $\hat{w}$ and optimize over a subset of variables and keep the other variables fixed to their corresponding values in the initial state $\hat{w}$. If the problem is feasible, we can update the solution as the initial state, and optimize over a new subset of variables. If the first selection of the subset resulted in an infeasible problem, one could increase the size of the subset of variables that are optimized over or simply select a new subset. However, in practice, we never encountered a case when fixing variables and optimizing over a subset resulted in an infeasible problem. Partly because there are relatively few constraints compared to the number of variables and partly because we still kept a large number of variables that we optimized over. 

This block coordinate optimization process is simple, but there is one aspect that deserves a more detailed discussion. Let $J_v$  denote the subset of variables that we optimize over in each step. When updating the subsets of variables $J_v$, one should be careful and ensure that there is an overlap with the previous subset of variables. This is described in the following proposition.
\begin{proposition}
Assume that the variables with indices $\{1, \dots, N\}\setminus J_v$ in \eqref{eq:QP-fixed} are fixed to their corresponding values in $\hat{w}$ and the resulting problem is feasible with a minimizer denoted by $\bar{w}$. Then by selecting any new subset of variables $J_v^+$, such that $J_v \cap J_v^+ = \emptyset$,
we cannot obtain a better solution by optimizing over the variables $J_v^+$ and fixing the rest to $\bar{w}$.
\end{proposition}
\begin{proof}
As $J_v^+ \subset \{1,\dots,N\}\setminus J_v$, it is clear that $\bar{w}_i = \hat{w}_i \ \ \forall i \in J_v^+$. As we assumed that the problem was feasible by optimizing over the variables $J_v$, we know that $\bar{w} \in \mathcal{L}$. Furthermore, with $w_i=\bar{w}_i$ fixed for all $i \in J_v$, $\bar{w}$ will remain the optimal solution when optimizing over $J_v^+$, as $w_i=\hat{w}_i ~\forall i \in J_v^+$ minimizes their contribution to the objective.
\end{proof}

From the proposition, it is clear that we need to keep an overlap between the subsets of variables we optimize; otherwise, we will get stuck in the first feasible solution found. The algorithm is presented as pseudo-code in Algorithm~\ref{alg:approximation_alternate}.

\begin{algorithm}
     \caption{Alternating Approximation Algorithm for QP \eqref{eq:QP-fixed}}\label{alg:approximation_alternate}
\KwData{The feasible set $\mathcal{L}$, the variable index $I:= \subset \{1,\dots,N\}$,  the fixed-ratio $0<p<1$, the maximum iterations number $T$, initial state $\hat{w} \in \R^{N}$\;}
\KwOut{$w$\;}
$J_0\leftarrow \emptyset$\;
$w^0 \leftarrow \hat{w}$\;
$J_v \leftarrow j_1,\dots,j_{(1-p)|I|} \sim U(I)$ \hfill \# Uniformly sample $(1-p)|I|$ indices from $I$\;
\For{$k=1,\dots,T$}
{   
    $w^k \leftarrow \argmin\{ \|w-\hat{w}\|^2_2 \mid w\in \mathcal{L}\subset \mathbb{R}^N, w_j=w^{k-1}_j,~ \forall j \notin J_v\cup J_0\}$\;
    $J\leftarrow I \setminus (J_v \cup J_0)$ \;
    $J_0 \leftarrow j_1,\dots, j_{(1-p)|I|/2} \sim U(J_v\cup J_0)$
    \hfill \#Keep half of the selected indices from $J$\;
    
    $J_v \leftarrow j_1,\dots, j_{(1-p)|I|/2} \sim U(J)$
    \hfill \#Re-select variables to optimize\;

}
\end{algorithm}
The selection of the fixing ratio is a trade-off between computational efficiency and approximation quality. Primary results show that even when fixing more than 50\% of variables, good solutions could be obtained after only a few iterations.

For the neural networks, we use a slightly more tailor-made selection strategy for selecting the variables to optimize in each step. We always keep the bias of each node as a variable, and we randomly select a proportion of variables in each layer. This ensures some flexibility to the optimization problem as the output of each node can be adjusted. 

For the larger neural network ResNet8, we tackle the complexity of \eqref{eq:QP-fixed} by fixing 80\% of the weight variables in each layer while keeping all the bias variables free. We performed an initial test varying the number of maximum iterations $T$, and in the end, decided on using $T=1$ considering that more iterations did not lead to any clear improvements. Although it could lead to better solutions to the QP problem \eqref{eq:QP-fixed}, it did not result in a final neural network with better properties. Keep in mind, for Algorithm~\ref{alg:model_update}, we only need a feasible solution to the QP problem \eqref{eq:QP-fixed}, and a solution with a slightly lower objective will not necessarily perform better in reducing $\ell(w)$ and $V(w,\mathcal{X}_\text{adv})$.

For the other neural networks, we did not use Algorithm~\ref{alg:approximation_alternate}, as we can directly solve the corresponding \eqref{eq:QP-fixed} using Gurobi.

\section{Computational experiments}
\label{sec:experiments}

Here, we test the performance of the proposed fine-tuning method Algorithm~\ref{alg:model_update} by applying it to two benchmark datasets with different complexity and to three neural networks with different architectures. The benchmark data sets are listed below:

\begin{itemize}
    \item MNIST \cite{deng2012mnist}: A dataset of handwritten digits of numbers 0 to 9, with each image of  28x28 pixels. The MNIST database contains 60,000 training images and 10,000 testing images. In the reprocessing, the pixel values of each image are scaled to between $0$ and $1$.
    \item CIFAR10 \cite{alex2009learning}: A dataset of 32x32 color images in 10 classes, \ie, airplanes, cars, birds, cats, deer, dogs, frogs, horses, ships, and trucks.
     There are 50,000 training images and 10,000 test images. 
    In the pre-processing, each image is 
    random cropped with reflection padding of  size $4$, flipped horizontally with probability of $0.5$,  and normalized with mean $(0.4914, 0.4822, 0.4465)$ and standard deviation $(0.2023, 0.1994, 0.2010)$ similar to the approach in \cite{kagglecifar10dataaug2019}.
\end{itemize}

\noindent It should be noted that CIFAR10 is commonly known to be a significantly more challenging benchmark test than MNIST.

We generate the adversarial dataset $\mathcal{X}_\text{adv}$ in the following way: First, we apply the projected gradient
descent (PGD) method on $\mathcal{X}_{\text{train}}$: Then, all the perturbed data are sorted in descending order according to their violation metrics  
$V(\cdot)$; Finally, and the first $|\mathcal{X}_{\text{adv}}|/10$ are selected for each label.
We implement the process by implemented package \textsc{CleverHans}~\citep{papernot2018cleverhans} with the following configurations similar as in \cite{madry2017towards}:
\begin{itemize}
    \item MNIST: run $50$ iterations of projected gradient descent as our adversary, with a step size of $0.01$, $\ell_{\infty}$-norm with radius $0.1$.
    \item CIFAR10: run $3$ iterations of projected gradient descent with a step size of $3/255$, $\ell_{\infty}$-norm with radius $8/255$.
\end{itemize}

To measure the robustness of neural networks, we test their accuracy on adversarial datasets generated by PGD, and the Fast Gradient Sign Method (FGSM)~\citep{goodfellow2014explaining} from CleverHans using the default test dataset. The configuration for PGD is the same as for the training process, and the configuration for FGSM has the same radius with  $\ell_{\infty}$ norm.

We choose to use different architectures for each dataset. 
For the MNIST dataset, we choose convolutional neural networks with two sizes: 
\begin{itemize}
\item CNNLight: A ConvNet architecture with a smaller size, which includes 2 convolutional layers, with 8 and 16 channels (each with a stride of 2, to decrease the resolution by half without requiring max pooling layers), and 2 fully connected layers stepping down to 50 and then 10 (the output dimension) hidden units, with ReLUs following each layer except the last. The total number of trainable parameters is 40908.
    \item CNN: A ConvNet architecture that includes 2 convolutional layers, with 16 and 32 channels (each with a stride of 2, to decrease the resolution by half without requiring max pooling layers), and 2 fully connected layers stepping down to 100 and then 10 (the output dimension) hidden units, with ReLUs following each layer except the last. The total number of trainable parameters is 162710. 
\end{itemize}

For CIFAR10, we choose the residual network due to the increasing complexity of classifying the images. Details on the neural network architecture:
\begin{itemize}
    \item ResNet8: An initial convolutional layer with 16 channels, followed by 3 groups of residual blocks. Those containing 2 convolutional layers have 16 and 32 channels (with a stride of 2 for downsampling), and the last one has  64 channels (also with a stride of 2 for downsampling). Finally, the model has an adaptive average pooling layer followed by a fully connected layer of 64x10. The residual connections are from conv1 to conv3, conv3 to conv5, and conv5 to  conv7. The total number of trainable parameters is 174970.
\end{itemize}
The neural networks are trained using standard procedures. The convolutional layer neural networks (\ie, CNNLight and CNN) are trained with 10 epochs, with a batch size of 64. The optimizer is Adaptive Moment Estimation (Adam) with an initial learning rate of 0.7 with a decaying factor of 0.7 and stepsize of 1. The ResNet8 models are trained with 10 epochs with a batch size of 64, and the optimizer is Stochastic Gradient Descent (SGD) with a learning rate of 0.01 and a momentum of 0.9.

As a reference point, Table~\ref{tab:train_with_adv} shows results obtained by directly adding the adversarial data $\mathcal{X}_{\text{adv}}$ to the training data. $\mathcal{X}_{\text{test}}$ denotes the test dataset, and FGSM and PGD denote the perturbed datasets generated by corresponding attackers. We also include the incremented percentage points in the brackets for comparison with the pre-trained models. The results here point out that a very small number of adversarial data points makes very little impact within a standard training framework.  
With $
|\mathcal{X}_{\text{adv}}|=50$, the models have a similar performance to the pre-trained model with default settings.
\begin{table}[htb!]
    \centering
    \begin{tabular}{rrrrr}
    \hline
    
         \multirow{2}{*}{Model} &   \multirow{2}{*}{Dataset}
         &\multicolumn{3}{c}{Accuracy (\%)}\\      
         \cline{3-5}
         && $\mathcal{X}_{\text{test}}$ & FGSM &  PGD \\
         \hline
         CNNLight & MNIST &  98.02 (+0.21) & 9.07 ($-$3.31) & 3.36 ($-$0.38) \\
         \hline
         CNN & MNIST &   98.74 (+0.19) & 15.98 (+0.31) &  5.53 (+1.51) \\
         \hline
         ResNet8 & CIFAR10 & 75.54 ($-$7.27) & 34.87 (+1.48) & 33.00 (+2.50) \\
         \hline
    \end{tabular}
    \caption{Performance of models trained with adversarial data ( $|\mathcal{X}_{\text{adv}}|=50$) added to the training. The numbers in parenthesis show the difference compared to the initial models. }
    \label{tab:train_with_adv}
\end{table}

\begin{figure}[htb!]
    \centering
    \subfloat[FGSM]{
    \includegraphics[width=0.47\linewidth]{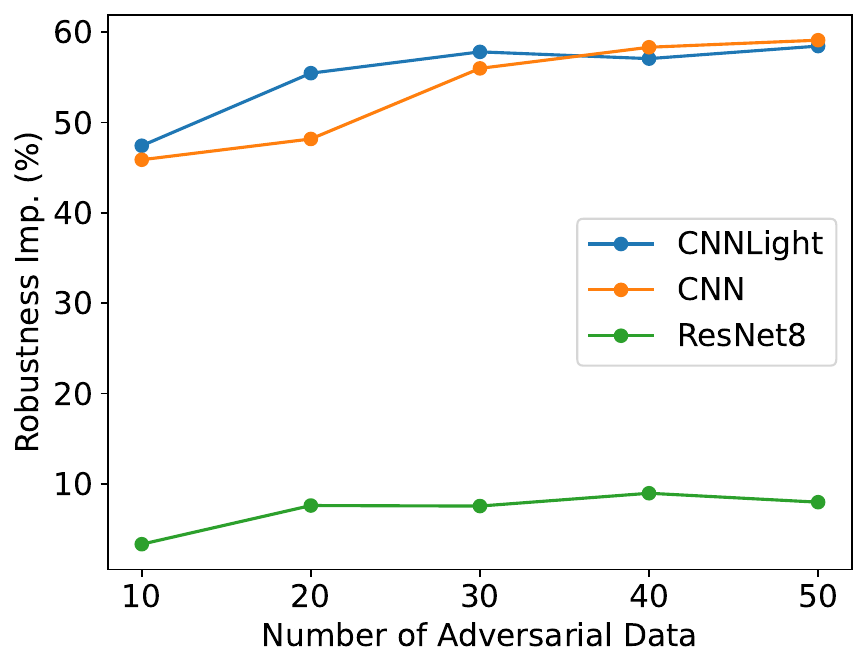}}
    \subfloat[PGD]{
    \includegraphics[width=0.47\linewidth]{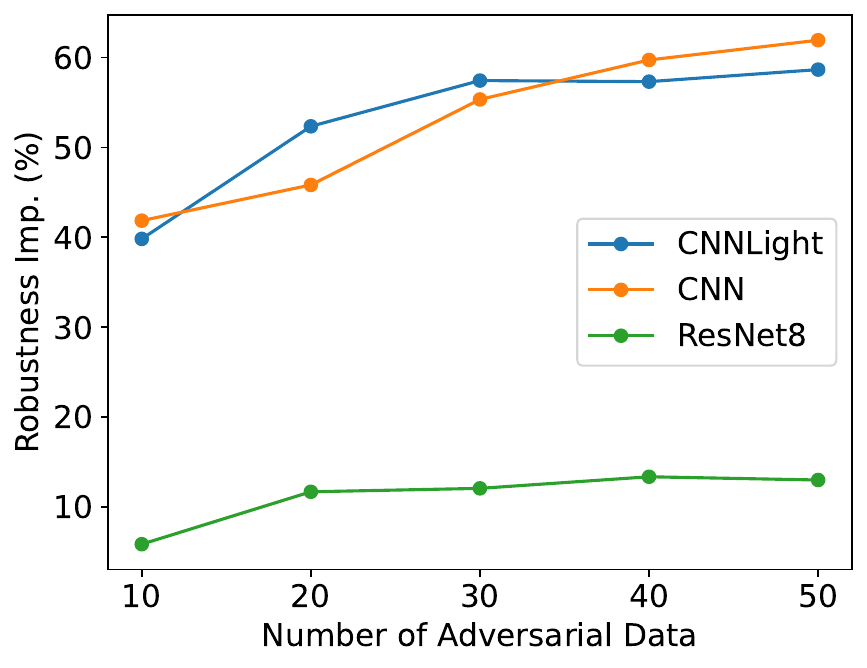}
    }
    \caption{Improvements of robustness in percentage points with $|\mathcal{X}_{adv}|=\{10,20,30,40,50\}$ using Algorithm~\ref{alg:model_update} with $M=2$.}
    \label{fig:acc_inc}
\end{figure}
Figure~\ref{fig:acc_inc} shows the increased percentage points for accuracy obtained by Algorithm~1 on adversarial data generated by FGSM and PGD approaches, while more details are given in later subsections. It can be seen that the improvements generally increase as more adversarial data are included in the fine-tuning process. 
CNNLight and CNN show a similar performance regarding the incremented improvements, where the fine-tuning process can achieve an improvement of more than 40 percentage points even with $|\mathcal{X}_{\text{adv}}|=10$ for both FGSM and PGD; when $|\mathcal{X}_{\text{adv}}|$ increases to $50$, the improvements increase up to 60 percentage points.
For ResNet8 (with CIFAR10), the improvement is smaller when only 10 adversarial data points are included, but the improvements increase up to 10 percentage points for FGSM and up to 15 points for PGD when more data are included. The improvement is less for ResNet8, but still clearly significant. Due to the more challenging dataset and the more complex model architecture, we would most likely need more adversarial data with ResNet8 to obtain larger improvements.

\subsection{Linearized input perturbations}
Here, we present some initial results obtained by the alternative linearizations that consider linearized perturbations to the adversarial data.
Figure~\ref{fig:epsilon_x} presents how $V(w,\mathcal{X}_{\text{adv}})$ and PGD accuracy change over the iterations of Algorithm~\ref{alg:model_update} with $\bar{\varepsilon}$ of $\{0, 0.05\}$.  The violation metrics drop very quickly from $\sim$250 to $\sim$25 in the first 3 iterations under both circumstances. Therefore, we only include the results from iterations 3-20 for better visualization in Figure~\ref{fig:vio_epsilon_x}. The solutions with $\bar{\varepsilon}=0.05$ always have less $V(w,\mathcal{X}_{\text{adv}})$ compared to the solutions with $\bar{\varepsilon}=0$. This indicates that a small perturbation in the adversary data (\ie, inequality ~\eqref{eq:robust_x2}) can help to reduce the violation regarding adversary correction constraints more effectively. 
In Figure~\ref{fig:pgd_epsilon_x},  the robustness against PGD attackers is similar for the two settings. This initial result shows a possibility of better performance with a smart selection of $\bar{\varepsilon}$. It is worth pointing out here that a through hyper-parameter tuning with varying values of $\bar{\varepsilon}$ can further improve the results, but we will limit the experiments with $\bar{\varepsilon}=0$ in the following section to illustrate the performance of Algorithm~\ref{alg:model_update}.

\begin{figure}
    \centering
\subfloat[$V(w,\mathcal{X}_{adv})$ over iteration 3 to 20\label{fig:vio_epsilon_x}]{\includegraphics[width=0.47\linewidth]{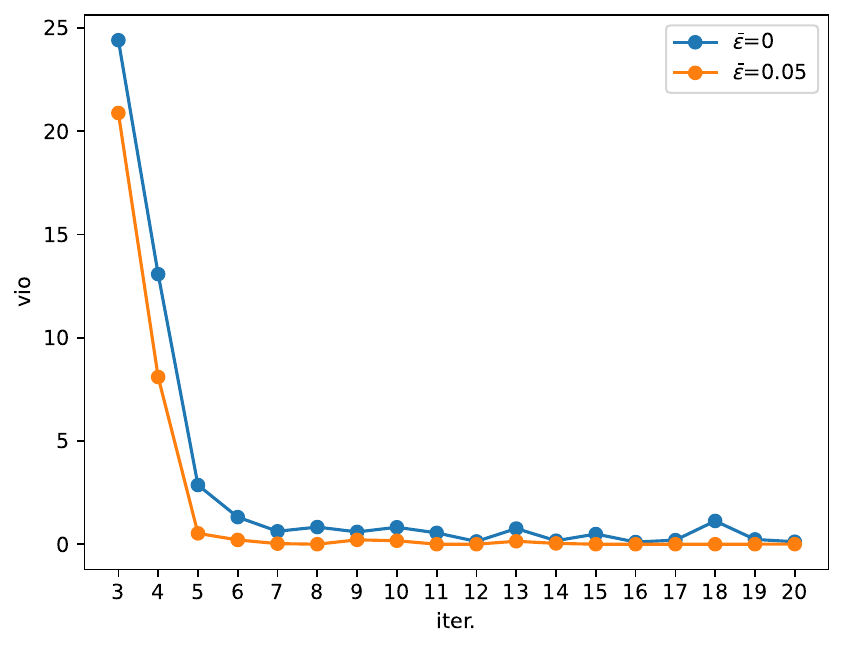}}
\subfloat[PGD accuracy over 20 iterations with $\omega = 0$\label{fig:pgd_epsilon_x}]{\includegraphics[width=0.47\linewidth]{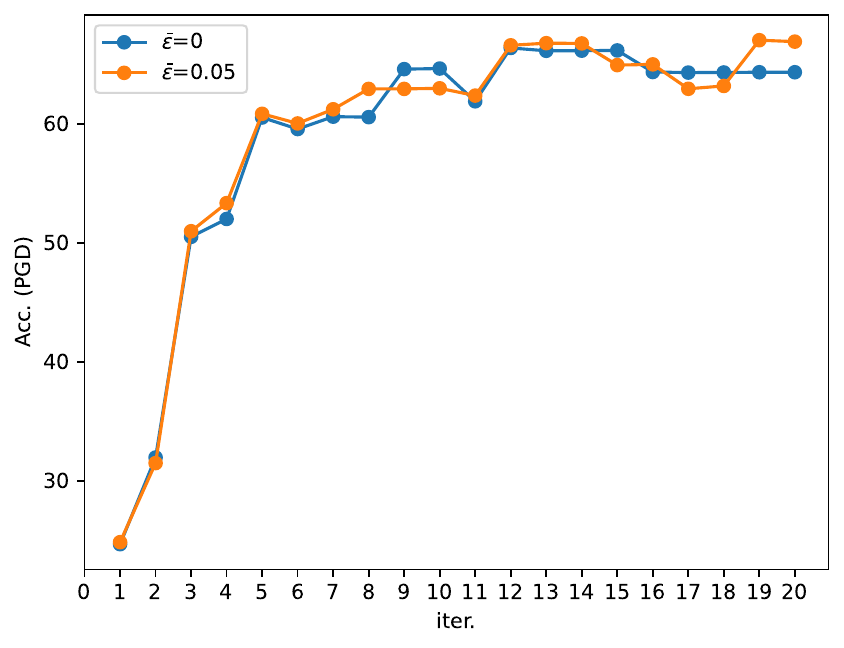}}
\caption{The comparison between $\bar{\varepsilon} = \{0, 0.05 \}$ with $|\mathcal{X}_{adv}|=50$, and $\omega =0$.}\label{fig:epsilon_x}
\end{figure}

\subsection{Accuracy for MNIST}
Table~\ref{tab:mnist_acc_cnnlight} and ~\ref{tab:mnist_acc} present the results of fine-tuned CNNLight and CNN on the MNIST dataset by \eqref{alg:model_update} with different configurations. We present the accuracy on test dataset $\mathcal{X}_{\text{test}}$ and the robustness (\ie, accuracy on FGSM and PGD) with the number of adversarial datapoints, \ie, $|\mathcal{X}_\text{adv}|$,  of $\{10,20,30,40,50\}$,  and $\omega$ of $\{0, 0.2, 0.4\}$.

All the experiments are set up with $M=20$, which is adequate to reduce $V(w,\mathcal{X}_\text{adv})$ below a threshold $0.5$ in preliminary tests for CNNLight and CNN. Both tables show that the loss on $\mathcal{X}_{\text{train}}$ and accuracy on $\mathcal{X}_{\text{test}}$ are kept close to the pre-trained model within 20 iterations.

In Table~\ref{tab:mnist_acc_cnnlight} and \ref{tab:mnist_acc}, the performance of the pre-trained model (\ie, $|\mathcal{X}_\text{adv}|=0$) is presented as a baseline.
The best result for each $|\mathcal{X}_\text{adv}|$ is marked in bold.

We present the results with $\omega $ of $ \{0, 0.2, 0.4\}$ to provide some insight into how it affects the results. In Table~\ref{tab:mnist_acc_cnnlight}, when $\omega= 0.4$, the loss of the fine-tuned model is slightly higher than the baseline, but the accuracy on $\mathcal{X}_{\text{test}}$ remains the same. While decreasing the $\omega$ from $0.4$ to $0.2$ or $0$, the accuracy on FGSM and PGD increases but the MNIST accuracy drops. Table~\ref{tab:mnist_acc} shows a similar pattern with various $\omega$, but a overall better performance under various circumstances.

\begin{table}[hbtp!]
    \centering
    \begin{tabular}{rrrrrr}
    \hline
\multirow{2}{*}{$|\mathcal{X}_\text{adv}|$} &\multirow{2}{*}{$\omega$}
&\multicolumn{2}{c}{MNIST} & \multicolumn{2}{c}{Adversary acc.(\%)}\\
  &  & loss   & acc.(\%) & FGSM  & PGD \\
\hline
0  &  & 0.0516 & 97.81 &	12.38 &	3.74\\
\hline
10&0.0 & 0.0644 & 97.52 & 58.02 & 42.77\\
&0.2 & 0.0603 &  97.68 &  \textbf{59.80} &  \textbf{43.56} \\
&0.4 & 0.0587& \textbf{97.69} & 57.09 & 37.95 \\
\hline
20&0.0&  0.0685  &  \textbf{97.65} & \textbf{67.83}	& \textbf{56.08} \\
&0.2 & 0.0670 &  97.51 & 66.67  & 55.55 \\
&0.4 & 0.0619 &  97.64 & 61.65 & 45.48 \\
\hline 
30 &0.0 & 0.0650& 97.70 &68.53	& 59.53\\
&0.2& 0.0611 & 97.68 & \textbf{70.19}  &  \textbf{61.17} \\
&0.4& 0.0579&  \textbf{97.79} & 65.55 & 52.06 \\
\hline
40& 0.0 &  0.0668&  97.66 & 69.35	& \textbf{61.05}\\
&0.2 & 0.0610  & 97.89 & \textbf{69.44} & 60.15 \\
&0.4 & 0.0573  & \textbf{97.92} & 64.96 & 50.99 \\
\hline 
50& 0 & 0.0622 & 97.82  &	\textbf{70.83} &	\textbf{62.39}  \\
&0.2 & 0.0602 & 97.90 &  69.10 &59.08\\
&0.4 & 0.0571 & \textbf{97.89} & 63.02 & 45.90\\
\hline
\end{tabular}
    \caption{Results from fine-tuning CNNLight with Algorithm~\ref{alg:model_update} and $M=20$.}
    \label{tab:mnist_acc_cnnlight}
\end{table}

\begin{table}[hbtp!]
    \centering
    \begin{tabular}{rrrrrr}
    \hline
\multirow{2}{*}{$|\mathcal{X}_\text{adv}|$} &\multirow{2}{*}{$\omega$}
&\multicolumn{2}{c}{MNIST} & \multicolumn{2}{c}{Adversary acc.(\%)}\\
  &  & loss   & acc.(\%) & FGSM  & PGD \\
\hline
0  &  & 0.0197 & 	98.55&	15.67 &	4.02 \\
\hline
10 & 0.0 & 0.0354  & 98.19 &	59.49&43.76\\
& 0.2 & 0.0269 &  98.46 & 61.54 &  45.73\\
&0.4 & 0.0269 & \textbf{98.46} & \textbf{61.54} & \textbf{45.87}\\
 
\hline
20 & 0.0 & 0.0276 & \textbf{98.54} &	\textbf{63.84}& 	\textbf{49.84} \\
& $0.2$ & 0.0276 & 98.54 &	63.84& 	49.84 \\
& $0.4$ &  0.0252 & 98.53 &	58.60 &	 39.71 \\
\hline
30& 0.0 &0.0296 & 98.31 & 70.92 & 58.87\\
& 0.2 & 0.0268 & 98.39 & \textbf{71.65} & \textbf{59.35} \\
& 0.4  &  0.0255  &\textbf{98.49}	&  65.65	& 48.58\\ 
\hline
40& 0.0 &0.0293 & 98.37 & 72.16 & 62.23 \\
&0.2 & 0.0289 &98.42  & \textbf{73.99} & \textbf{63.74} \\
&0.4 &0.0260 &\textbf{98.55}&	68.47&	53.98 \\
\hline 
50&0.0 & 0.0310 & 98.44& 73.79	& 64.33 \\
&0.2 & 0.0290  & 98.46 & \textbf{74.77} & \textbf{65.94} \\
&0.4 & 0.0262 & \textbf{98.52} &71.24 &  57.70\\

\hline 
\end{tabular}
    \caption{Results from fine-tuning CNN with Algorithm~~\ref{alg:model_update} and $M=20$.}
    \label{tab:mnist_acc}
\end{table}

\subsection{Accuracy on CIFAR10}
Table~\ref{tab:cifar10_resnet8} presents the fine-tuning results of ResNet8 on CIFAR10. Compared to the MNIST dataset, CIFAR10 is a more complicated dataset, and a larger model is needed to achieve a decent accuracy.
The pre-trained ResNet8 has a prediction accuracy of 82.79\% on CIFAR10, and 33.39 \% (resp. 30.50 \%) on adversarial data generated by FGSM (resp. PGD). 
After the fine-tuning process with $|\mathcal{X}_\text{adv}|$ up to 50, the accuracy on FGSM and PGD increase up to 42.35 \% and  43.86\%.

Similarly, as for CNNLight and CNN, even using only 10 adversarial data can achieve a noticeable improvement in adversarial data classification (see Table~\ref{tab:cifar10_resnet8}). Even though the growing pace is slower, the improvement is further enlarged when $\mathcal{X}_{adv}$ is included in the fine-tuning process. The trend is that more adversarial data results in improved robustness, but the improvements are smaller than we obtained for MNIST. As the model is larger and, in general, more difficult to train on the data set, we might need more adversarial data and extra iterations (\ie, $M>20$) to obtain larger improvements. But, the computational burden is noticeable when fine-tuning ResNet8 models. We chose not to run more iterations using more data, as the results clearly show that concepts work and can improve robustness.

\begin{table}[hbtp!]
    \centering
\begin{tabular}{rrrrrr}
    \hline
\multirow{2}{*}{$|\mathcal{X}_\text{adv}|$} &\multirow{2}{*}{$\omega$}
&\multicolumn{2}{c}{CIFAR10} & \multicolumn{2}{c}{Adversary acc.(\%)}\\
  &  & loss   & acc.(\%) & FGSM  & PGD \\
\hline
0  &  & 0.44 & 82.79	& 33.39	&30.50 \\
\hline
10 & 0.0  & 0.46	& \textbf{82.71} & 36.31 & 35.97\\
&0.2 &  0.46 & 82.20 & \textbf{36.71} & \textbf{36.36}  \\
& 0.4 &  0.44 & 82.64 &  35.79 & 34.95 \\
\hline
20&0.0 & 0.49  & 81.49 &	38.12 &39.20 \\
&0.2 & 0.46  & \textbf{81.92} & \textbf{40.99} & \textbf{42.19}\\
&0.4 & 0.46 &  82.57 &  40.14 & 41.25\\
\hline 
30&0.0 & 0.50 & 81.24& \textbf{40.93}&  \textbf{42.57}\\
&0.2 &  0.46 & \textbf{82.19}& 39.67 &40.52\\
&0.4 &  0.46 & 82.09 & 40.84 &42.22\\
\hline
40 &0.0 & 0.47 & 82.18 & \textbf{42.35} &	\textbf{43.86}  \\
&0.2 &0.46& 82.56 & 41.29 & 43.26 \\
&0.4 & 0.46 & \textbf{82.66} & 41.59 & 43.18 \\
\hline 
50&0.0  & 0.46& \textbf{82.34} & \textbf{41.36} &\textbf{43.50} \\
&0.2 & 0.46 & 82.12 & 41.12 & 43.12\\
&0.4 & 0.45 & 83.16 &  39.69 & 40.86\\
\hline
\end{tabular}
    \caption{Results from fine-tuning ResNet8 with Algorithm~\ref{alg:model_update} and $M=20$.}
    \label{tab:cifar10_resnet8}
\end{table}

\subsection{Computational Efficiency}
The fine-tuning process Algorithm~\ref{alg:model_update} includes iteratively solving a growing QP problem, evaluating intermediate neural networks, and extracting the gradient information in order to generate auxiliary linear constraints. Table~\ref{tab:run_time} presents the running time of solving each QP problem under various conditions, and the evaluation time of computing loss on the training dataset for neural networks with different architectures, with the aim of providing insights into the performance of each component. 
We break down the running time into two categories (\ie, optimization-related and evaluation-related) to further analyze the computational efficiency of this algorithm. The optimization-related computation consists of the time it takes to construct the QP problems and solve them using Gurobi. The evaluation-related computation refers to the evaluation of neural networks in order to access gradient information. 
\begin{figure}
  \centering
   \subfloat[$|\mathcal{X}_\text{adv}|=10$]{\includegraphics[width=0.47\textwidth]{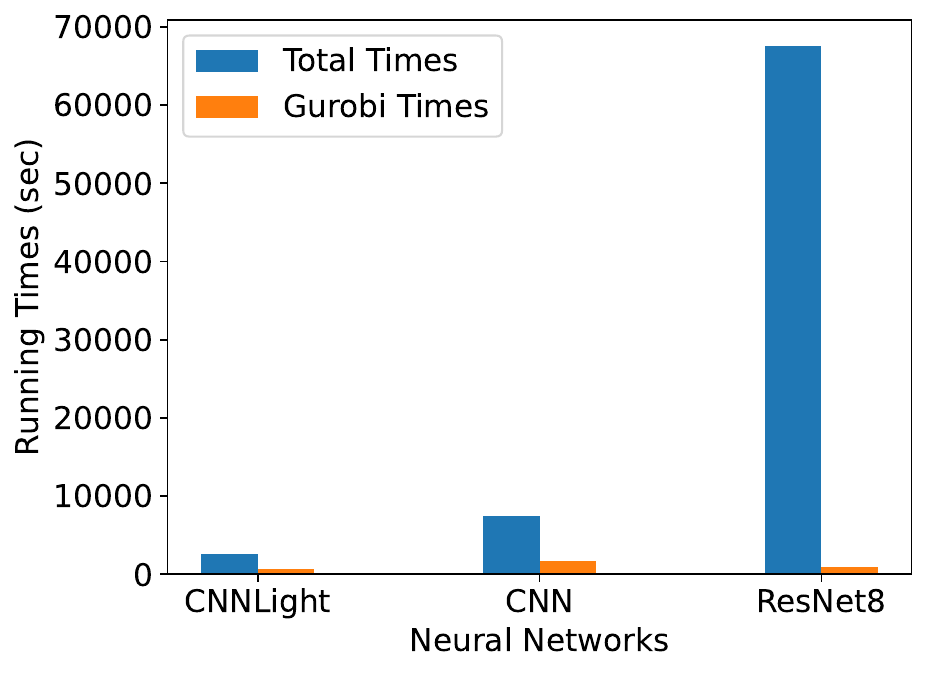}}
    \subfloat[$|\mathcal{X}_\text{adv}|=50$]{\includegraphics[width=0.47\textwidth]{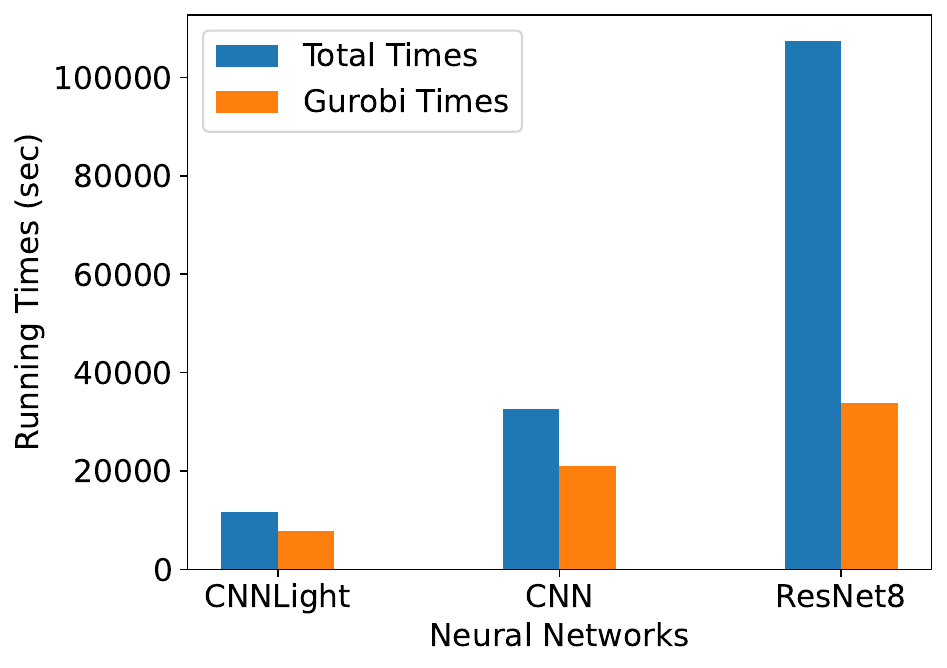}}
    \caption{This figure shows the total running time and the time spent on solving the QP problems \eqref{eq:qp_cut_iter} (\ie, running time of Gurobi).}\label{tab:running_time2}
\end{figure}
Figure~\ref{tab:running_time2} illustrates the total fine-tuning time and optimization-related time (\eg, Gurobi running times) for different neural network architectures. For CNNLight, the total fine-tuning time is between 2623.46 sec and 11763.83 sec where the number of adversarial datapoints (\ie, $|\mathcal{X}_\text{adv}|$) are between 10 and 50. The total optimization-related time lies between 624.37 sec and 7830.74 sec.
For CNN, the total running time is between 7501.90 sec and 32514.06 sec for $|\mathcal{X}_\text{adv}|$ between 10 and 50, and the total optimization-related time is between 1699.61 sec and 21051.21 sec.
For ResNet8, due to the increasing size of the neural network, the total running is between 67518.09 sec to 107319.27 sec with $|\mathcal{X}_\text{adv}|$ from 10 to 50, where the total optimization-related time is between  979.31 sec to 33763.22 sec respectively.

The above results are from experiments conducted only on CPU, and they shown that the optimization-related computations only accounts for 20\% - 65\% of the total execution time for CNNs, and accounts for even less, 1\% -  31\%, when considering the larger neural networks, \ie, ResNet8.

We want to stress that the computations could be sped up significantly. Many of the computations can easily be run in parallel, as we discuss more in detail in the following subsections. Additionally, we have not used the best machine for running experiments of this type. As an example, the time for evaluating the neural networks could potentially be reduced by a factor of 10 (as the main evaluations could be run in 10 parallel processes). We have not investigated this further, as our main goal with this paper is to provide a proof of concept -- that this approach can successfully be used in order to improve the robustness of neural networks. We believe the numerical results clearly show the feasibility and value of the proposed method.

\subsubsection*{The optimization-related computation}
The difficulty of solving QP problems directly relates to the number of constraints and variables. In our setting, the number of variables is proportional to the size of the neural network. The number of constraints is proportional to the size of $\mathcal{W}_\text{cut}$ and $\mathcal{X}_\text{adv}$, which is reflected by the neural network type, iteration number $k$ and $|\mathcal{X}_\text{adv}|$ respectively in Table~\ref{tab:run_time}.

It is shown that the QP solving times are similar for different architectures with varying $|\mathcal{X}_\text{adv}|$, but are more impacted by the size of the neural network when $k=20$. This pattern indicates the bottleneck of solving QPs efficiently is caused by the size of the neural network.  
Due to the nature of large-sized neural networks and the formulation of \eqref{eq:qp_cut_iter}, the number of variables is relatively large compared to the number of constraints. Therefore, we believe it could be solved more efficiently with a tailor-made implementation of a dual algorithm that could also benefit from warm-starts. We have not investigated this further, but we believe the solution times could be sped up significantly.

\subsubsection*{The evaluation-related computation}
The evaluation times of the loss function on the training dataset for CNNLight, CNN, and ResNet8 are shown in Table~\ref{tab:run_time} to reflect how the level of evaluation-related computational cost changes with the size of neural networks grows. We emphasize here that the execution times of evaluating the neural networks depend on the total number of iterations and the partition of stepsizes for the line search. 

For CNNLight and CNN, the loss evaluation time for a single execution is relatively small. Still, the accumulated time -- considering a total of 20 iterations and 10 line search stepsizes -- is still significant. To improve the efficiency, a more finely partitioned line search could be applied only in later iterations, since the high-quality candidate solutions are not found in the early stages. Moreover, the evaluation of the neural networks can be parallelized for better computational efficiency, especially for ResNet8, where a single execution of evaluating the loss can take $\sim$ 270 sec. All points of the line search could be evaluated fully in parallel, and this could greatly speed up the computations. We could also evaluate the loss function on smaller batches to reduce computational cost in this category. Thus, the evaluation times could be reduced greatly in a more sophisticated (parallel) implementation.

\begin{table}[hbtp!]
\centering
    \begin{tabular}{r|rrr|rr}
\hline
    & \multicolumn{3}{c|}{QP (Gurobi)} & \multicolumn{2}{c}{Eval. of loss}  \\
    &   $|\mathcal{X}_\text{adv}|$ & $k$ &    time  & Dataset & time  \\
    \hline 
CNNLight  &  10 & 1  &1.10 & \multirow{4}{*}{MNIST}  &\multirow{4}{*}{$\sim$    0.07}   \\
        &  10 & 20 &  69.23 \\
        &  50 & 1 & 10.66   \\
        & 50 & 20 & 949.78 \\
    \hline
    CNN & 10& 1&1.13 &\multirow{4}{*}{MNIST} & \multirow{4}{*}{$\sim$ 9.39}\\
     & 10 & 20 & 69.24 & \\
     & 50 & 1 & 30.19 & \\
     & 50 & 20 &3089.63 & \\
    \hline
    ResNet8 &10 & 1& 2.09&\multirow{4}{*}{CIFAR10}&\multirow{4}{*}{$\sim$ 275.92}\\
    & 10 & 20 & 96.53 &  \\
    & 50 & 1 & 15.98 &\\
    & 50 & 20 & 4130.36 & \\ 
    
    \hline
    \end{tabular}
    \caption{Running times of the main operations in Algorithm~\ref{alg:model_update}. $k$ is the iteration index, and in the later iterations, the QP problems \eqref{eq:qp_cut_iter} will be more difficult due to more constraints.}\label{tab:run_time}
\end{table}

\section{Conclusion}
We reformulate the robust training problem as a fine-tuning process that aims to solve an adversary-correction problem. This nonlinear problem is solved by linearization and a cutting-plane framework to improve the approximation. Based on this framework, Algorithm~\ref{alg:model_update} uses line search and a weighted objective function to balance the accuracy on both the test and adversarial datasets. 
We illustrate the performance of our method by numerical experiments with pre-trained models of various sizes and architectures. The numerical results show that we can significantly improve the robustness of CNNs and ResNet8 even with as few as 10 adversarial data points in the fine-tuning process and obtain even more significant improvements using more adversarial data. Our method shows a new direction for robust training of neural networks. Both the theory and the numerical results support that the adversarial correction problem is a sound approach for fine-tuning neural networks to improve robustness.

\section{Acknowledgment}
The authors gratefully acknowledge funding from Digital Futures and the 3.ai Digital Transformation Institute project \say{AI techniques for Power Systems Under Cyberattacks}.

\ifnum \Journaltemplate=1
\bibliography{mybib}
\else
\bibliographystyle{apalike}
\bibliography{mybib}

\begin{thebibliography}{}

\bibitem[Alex, 2009]{alex2009learning}
Alex, K. (2009).
\newblock Learning multiple layers of features from tiny images.
\newblock {\em https://www. cs. toronto. edu/kriz/learning-features-2009-TR. pdf}.

\bibitem[Aspman et~al., 2024]{Aspman_Korpas_Marecek_2024}
Aspman, J., Korpas, G., and Marecek, J. (2024).
\newblock Taming binarized neural networks and mixed-integer programs.
\newblock {\em Proceedings of the AAAI Conference on Artificial Intelligence}, 38(10):10935--10943.

\bibitem[Botoeva et~al., 2020]{botoeva2020efficient}
Botoeva, E., Kouvaros, P., Kronqvist, J., Lomuscio, A., and Misener, R. (2020).
\newblock Efficient verification of relu-based neural networks via dependency analysis.
\newblock In {\em Proceedings of the AAAI Conference on Artificial Intelligence}, volume~34, pages 3291--3299.

\bibitem[Deng, 2012]{deng2012mnist}
Deng, L. (2012).
\newblock The mnist database of handwritten digit images for machine learning research.
\newblock {\em IEEE Signal Process. Mag.}, 29(6):141--142.

\bibitem[Eichfelder, 2021]{eichfelder2021twenty}
Eichfelder, G. (2021).
\newblock Twenty years of continuous multiobjective optimization in the twenty-first century.
\newblock {\em EURO Journal on Computational Optimization}, 9:100014.

\bibitem[Elek, 2019]{kagglecifar10dataaug2019}
Elek, M.~I. (2019).
\newblock Data augmentation with keras using cifar-10.

\bibitem[Eykholt et~al., 2018]{eykholt2018robust}
Eykholt, K., Evtimov, I., Fernandes, E., Li, B., Rahmati, A., Xiao, C., Prakash, A., Kohno, T., and Song, D. (2018).
\newblock Robust physical-world attacks on deep learning visual classification.
\newblock In {\em Proceedings of the IEEE conference on computer vision and pattern recognition}, pages 1625--1634.

\bibitem[Fischetti and Jo, 2018]{fischetti2018deep}
Fischetti, M. and Jo, J. (2018).
\newblock Deep neural networks and mixed integer linear optimization.
\newblock {\em Constraints}, 23(3):296--309.

\bibitem[Goodfellow et~al., 2014]{goodfellow2014explaining}
Goodfellow, I.~J., Shlens, J., and Szegedy, C. (2014).
\newblock Explaining and harnessing adversarial examples.
\newblock {\em arXiv preprint arXiv:1412.6572}.

\bibitem[{Gurobi Optimization, LLC}, 2024]{gurobi}
{Gurobi Optimization, LLC} (2024).
\newblock {Gurobi Optimizer Reference Manual}.

\bibitem[Hassibi et~al., 1993]{hassibi1993optimal}
Hassibi, B., Stork, D.~G., and Wolff, G.~J. (1993).
\newblock Optimal brain surgeon and general network pruning.
\newblock In {\em IEEE international conference on neural networks}, pages 293--299. IEEE.

\bibitem[Huang et~al., 2018]{huang2018safety}
Huang, X., Kroening, D., Kwiatkowska, M., Ruan, W., Sun, Y., Thamo, E., Wu, M., and Yi, X. (2018).
\newblock Safety and trustworthiness of deep neural networks: A survey.
\newblock {\em arXiv preprint arXiv:1812.08342}, page 151.

\bibitem[Katz et~al., 2017]{katz2017reluplex}
Katz, G., Barrett, C., Dill, D.~L., Julian, K., and Kochenderfer, M.~J. (2017).
\newblock Reluplex: An efficient smt solver for verifying deep neural networks.
\newblock In {\em Computer Aided Verification: 29th International Conference, CAV 2017, Heidelberg, Germany, July 24-28, 2017, Proceedings, Part I 30}, pages 97--117. Springer.

\bibitem[Kelley, 1960]{kelley1960cutting}
Kelley, Jr, J.~E. (1960).
\newblock The cutting-plane method for solving convex programs.
\newblock {\em Journal of the society for Industrial and Applied Mathematics}, 8(4):703--712.

\bibitem[Kurakin et~al., 2016]{kurakin2016adversarial}
Kurakin, A., Goodfellow, I., and Bengio, S. (2016).
\newblock Adversarial machine learning at scale.
\newblock {\em arXiv preprint arXiv:1611.01236}.

\bibitem[LeCun et~al., 2015]{lecun2015deep}
LeCun, Y., Bengio, Y., and Hinton, G. (2015).
\newblock Deep learning.
\newblock {\em nature}, 521(7553):436--444.

\bibitem[LeCun et~al., 1989]{lecun1989optimal}
LeCun, Y., Denker, J., and Solla, S. (1989).
\newblock Optimal brain damage.
\newblock {\em Advances in neural information processing systems}, 2.

\bibitem[Liao et~al., 2018]{liao2018defense}
Liao, F., Liang, M., Dong, Y., Pang, T., Hu, X., and Zhu, J. (2018).
\newblock Defense against adversarial attacks using high-level representation guided denoiser.
\newblock In {\em Proceedings of the IEEE conference on computer vision and pattern recognition}, pages 1778--1787.

\bibitem[Liu et~al., 2021]{liu2021training}
Liu, A., Liu, X., Yu, H., Zhang, C., Liu, Q., and Tao, D. (2021).
\newblock Training robust deep neural networks via adversarial noise propagation.
\newblock {\em IEEE Transactions on Image Processing}, 30:5769--5781.

\bibitem[Machado et~al., 2021]{machado2021adversarial}
Machado, G.~R., Silva, E., and Goldschmidt, R.~R. (2021).
\newblock Adversarial machine learning in image classification: A survey toward the defender’s perspective.
\newblock {\em ACM Computing Surveys (CSUR)}, 55(1):1--38.

\bibitem[Madry et~al., 2017]{madry2017towards}
Madry, A., Makelov, A., Schmidt, L., Tsipras, D., and Vladu, A. (2017).
\newblock Towards deep learning models resistant to adversarial attacks.
\newblock {\em arXiv preprint arXiv:1706.06083}.

\bibitem[M{\"a}kel{\"a}, 2002]{makela2002survey}
M{\"a}kel{\"a}, M. (2002).
\newblock Survey of bundle methods for nonsmooth optimization.
\newblock {\em Optimization methods and software}, 17(1):1--29.

\bibitem[Papernot et~al., 2018]{papernot2018cleverhans}
Papernot, N., Faghri, F., Carlini, N., Goodfellow, I., Feinman, R., Kurakin, A., Xie, C., Sharma, Y., Brown, T., Roy, A., Matyasko, A., Behzadan, V., Hambardzumyan, K., Zhang, Z., Juang, Y.-L., Li, Z., Sheatsley, R., Garg, A., Uesato, J., Gierke, W., Dong, Y., Berthelot, D., Hendricks, P., Rauber, J., and Long, R. (2018).
\newblock Technical report on the cleverhans v2.1.0 adversarial examples library.
\newblock {\em arXiv preprint arXiv:1610.00768}.

\bibitem[Papernot et~al., 2017]{papernot2017practical}
Papernot, N., McDaniel, P., Goodfellow, I., Jha, S., Celik, Z.~B., and Swami, A. (2017).
\newblock Practical black-box attacks against machine learning.
\newblock In {\em Proceedings of the 2017 ACM on Asia conference on computer and communications security}, pages 506--519.

\bibitem[Paszke et~al., 2017]{paszke2017automatic}
Paszke, A., Gross, S., Chintala, S., Chanan, G., Yang, E., DeVito, Z., Lin, Z., Desmaison, A., Antiga, L., and Lerer, A. (2017).
\newblock Automatic differentiation in pytorch.
\newblock In {\em NIPS-W}.

\bibitem[Qin et~al., 2019]{qin2019adversarial}
Qin, C., Martens, J., Gowal, S., Krishnan, D., Dvijotham, K., Fawzi, A., De, S., Stanforth, R., and Kohli, P. (2019).
\newblock Adversarial robustness through local linearization.
\newblock {\em Advances in neural information processing systems}, 32.

\bibitem[Raghunathan et~al., 2018]{raghunathan2018certified}
Raghunathan, A., Steinhardt, J., and Liang, P. (2018).
\newblock Certified defenses against adversarial examples.
\newblock In {\em International Conference on Learning Representations}.

\bibitem[Ruan et~al., 2018]{ruan2018reachability}
Ruan, W., Huang, X., and Kwiatkowska, M. (2018).
\newblock Reachability analysis of deep neural networks with provable guarantees.
\newblock {\em arXiv preprint arXiv:1805.02242}.

\bibitem[Sehwag et~al., 2020]{sehwag2020hydra}
Sehwag, V., Wang, S., Mittal, P., and Jana, S. (2020).
\newblock Hydra: Pruning adversarially robust neural networks.
\newblock {\em Advances in Neural Information Processing Systems}, 33:19655--19666.

\bibitem[Shi et~al., 2021]{shi2021fast}
Shi, Z., Wang, Y., Zhang, H., Yi, J., and Hsieh, C.-J. (2021).
\newblock Fast certified robust training with short warmup.
\newblock {\em Advances in Neural Information Processing Systems}, 34:18335--18349.

\bibitem[Sosnin et~al., 2024]{sosnin2024certified}
Sosnin, P., M{\"u}ller, M.~N., Baader, M., Tsay, C., and Wicker, M. (2024).
\newblock Certified robustness to data poisoning in gradient-based training.
\newblock arXiv:2406.05670.

\bibitem[Szegedy et~al., 2014]{szegedy2014intriguing}
Szegedy, C., Zaremba, W., Sutskever, I., Bruna, J., Erhan, D., Goodfellow, I., and Fergus, R. (2014).
\newblock Intriguing properties of neural networks.
\newblock {\em arXiv preprint arXiv:1312.6199}.

\bibitem[Thorbjarnarson and Yorke-Smith, 2023]{thorbjarnarson2023optimal}
Thorbjarnarson, T. and Yorke-Smith, N. (2023).
\newblock Optimal training of integer-valued neural networks with mixed integer programming.
\newblock {\em Plos one}, 18(2):e0261029.

\bibitem[Toro~Icarte et~al., 2019]{toro2019training}
Toro~Icarte, R., Illanes, L., Castro, M.~P., Cire, A.~A., McIlraith, S.~A., and Beck, J.~C. (2019).
\newblock Training binarized neural networks using mip and cp.
\newblock In {\em Principles and Practice of Constraint Programming: 25th International Conference, CP 2019, Stamford, CT, USA, September 30--October 4, 2019, Proceedings 25}, pages 401--417. Springer.

\bibitem[Tram{\`e}r et~al., 2017]{tramer2017ensemble}
Tram{\`e}r, F., Kurakin, A., Papernot, N., Goodfellow, I., Boneh, D., and McDaniel, P. (2017).
\newblock Ensemble adversarial training: Attacks and defenses.
\newblock {\em arXiv preprint arXiv:1705.07204}.

\bibitem[Tsay et~al., 2021]{tsay2021partition}
Tsay, C., Kronqvist, J., Thebelt, A., and Misener, R. (2021).
\newblock Partition-based formulations for mixed-integer optimization of trained relu neural networks.
\newblock {\em Advances in neural information processing systems}, 34:3068--3080.

\bibitem[Tsiligkaridis and Roberts, 2022]{tsiligkaridis2022understanding}
Tsiligkaridis, T. and Roberts, J. (2022).
\newblock Understanding and increasing efficiency of frank-wolfe adversarial training.
\newblock In {\em Proceedings of the IEEE/CVF Conference on Computer Vision and Pattern Recognition}, pages 50--59.

\bibitem[Wu and Zhang, 2021]{wu2021tightening}
Wu, Y. and Zhang, M. (2021).
\newblock Tightening robustness verification of convolutional neural networks with fine-grained linear approximation.
\newblock In {\em Proceedings of the AAAI Conference on Artificial Intelligence}, volume~35, pages 11674--11681.

\bibitem[Yu et~al., 2022]{yu2022combinatorial}
Yu, X., Serra, T., Ramalingam, S., and Zhe, S. (2022).
\newblock The combinatorial brain surgeon: pruning weights that cancel one another in neural networks.
\newblock In {\em International Conference on Machine Learning}, pages 25668--25683. PMLR.

\bibitem[Zhao et~al., 2023]{zhao2023model}
Zhao, S., Tsay, C., and Kronqvist, J. (2023).
\newblock Model-based feature selection for neural networks: A mixed-integer programming approach.
\newblock In {\em International Conference on Learning and Intelligent Optimization}, pages 223--238. Springer.

\bibitem[Z{\"u}gner and G{\"u}nnemann, 2019]{zugner2019certifiable}
Z{\"u}gner, D. and G{\"u}nnemann, S. (2019).
\newblock Certifiable robustness and robust training for graph convolutional networks.
\newblock In {\em Proceedings of the 25th ACM SIGKDD International Conference on Knowledge Discovery \& Data Mining}, pages 246--256.

\end{thebibliography}
\fi
\end{document}